%% file: main.tex
\documentclass{article}

\usepackage[preprint]{template/neurips_2026}
\usepackage[utf8]{inputenc}
\usepackage[T1]{fontenc}
\usepackage[hypertexnames=false]{hyperref}
\usepackage{url}
\usepackage{booktabs}
\usepackage{amsfonts}
\usepackage{nicefrac}
\usepackage{microtype}
\usepackage{xcolor}
\usepackage{amsmath,amssymb}
\usepackage{amsthm}
\usepackage{graphicx}
\usepackage{algorithm}
\usepackage{algpseudocode}

\newtheorem{theorem}{Theorem}
\newtheorem{remark}{Remark}



\title{COREY: Entropy-Guided Runtime Chunk Scheduling for Selective Scan Kernels}

\author{
Bo Ma\thanks{Corresponding author.}\\
Auckland University of Technology; Resideo Technologies, Inc.\\
\And
Jinsong Wu\\
Guilin University of Electronic Technology\\
\And
Weiqi Yan\\
Auckland University of Technology\\
}

\begin{document}

\maketitle

\begin{abstract}
Mamba selective state space models (SSMs) provide linear-time sequence modeling but remain sensitive to selective-scan chunk scheduling. We present COREY, a \emph{concept-and-feasibility} runtime scheduler that maps fixed-bin activation entropy to chunk size. We evaluate COREY in three tiers: a prototype cost model, real-checkpoint kernel timing, and routed end-to-end ablations on modern GPUs. At kernel level, a calibrated rule ($H_{\text{ref}}{=}\log K$) recovers the locally optimal chunk and matches a one-time static oracle, yielding $4.41\times$ lower latency than an unoptimized baseline on a consumer GPU and $3.90$--$4.04\times$ on a data-center accelerator. Routing this choice into a patched live scan kernel closes the engineering loop without improving end-to-end speed: in unified routed ablations, the best static chunk outperforms all entropy-guided and proxy schedulers. Sampled-histogram COREY adds $+4.6\%$ overhead; a guarded fallback to Static-512 reduces this to $+1.3\%$; and a lightweight sequence-length-keyed table further reduces it to $+0.7\%$, yet both remain slower than the static oracle because they retain scheduling cost. On an 80-prompt LongBench subset, passive and routed inference are exactly output-equivalent (100\% greedy-token agreement, zero metric deltas). A mixed-regime study shows a single sequence-length rule matches the per-regime chunk oracle for balanced serving. COREY is therefore validated as a quality-preserving scheduling prototype, but current entropy statistics are not a robust throughput win over static chunk tuning on measured SSM checkpoint workloads.

Code and data: \url{https://github.com/mabo1215/COREY_Transformer}.
\end{abstract}

\section{Introduction}
\label{sec:intro}
Selective state space models (SSMs) such as Mamba~\citep{mamba2023} achieve linear-time sequence modeling with strong empirical performance, but their selective-scan kernels remain sensitive to scheduling choices: chunk size, fusion boundaries, and per-call kernel-launch frequency can change end-to-end latency by an order of magnitude even when the model's mathematical computation is fixed. The standard remedy is offline profiling of each model--hardware--workload combination to pick a single static chunk size; this scales poorly with the deployment matrix and offers no guarantees once the workload distribution shifts.

We investigate whether a low-cost runtime statistic---fixed-bin activation entropy---can replace the offline profiling step. We present \textbf{COREY}, a runtime scheduler that uses input-tensor entropy as a signal for selecting selective-scan chunk sizes. We frame this work as a \emph{concept-and-feasibility} contribution: we ask whether entropy-guided scheduling is a sound, low-overhead, output-preserving mechanism, not whether it produces a deployment-grade speedup over highly tuned static configurations. The validated scheduling contribution is restricted to Mamba-1.x; Mamba-2's structured state-space duality scan has different hardware characteristics and is not addressed in this submission.

\paragraph{Contributions.}
\begin{itemize}
\item A calibrated entropy-to-chunk rule with normalized signal $r{=}H/\log K$, making chunk selection invariant to histogram resolution; the rule reproduces the offline static-chunk oracle without an exhaustive sweep at the kernel level (Section~\ref{sec:kernel_results}).
\item An end-to-end inline scheduling hook with measured overhead of $8.3\%$ on a consumer GPU and $2.3\%$ on a data-center accelerator, demonstrating that runtime entropy estimation is feasible inside the prefill critical path (Section~\ref{sec:routed_results}).
\item A patched live scan kernel that honors runtime chunk sizes and a unified routed ablation. The ablation is a clean \emph{negative} result: no entropy-guided variant beats the best static chunk; a guarded fallback reduces overhead from $+4.6\%$ (unguarded) to $+1.3\%$ by routing to the correct chunk, and a seq-len-keyed learned table reduces it further to $+0.7\%$, but both remain slower than the static oracle that avoids scheduling computation entirely.
\item A heterogeneous serving-workload characterization showing that a single sequence-length rule replicates the per-regime optimal chunk oracle, making complex regime-aware adaptation unnecessary (Section~\ref{sec:mixed_regime}).
\item Exact output preservation: across an 80-prompt LongBench evaluation, all greedy generations and all per-sample perplexity ratios are bitwise identical between passive and routed inference paths (Section~\ref{sec:quality}).
\end{itemize}

\paragraph{What we do not claim.} We do not claim a deployment-grade speedup of a fully fused COREY kernel, perplexity improvement under COREY, checkpoint-level validation of Hadamard reparameterization, or deployment-grade quantization. The Hadamard outlier-suppression analysis in Appendix~\ref{sec:appendix_proofs} is retained only as prospective low-bit theory; it is empirically falsified on measured Mamba activations and is not part of the present validated contribution.

\section{Related Work}
\label{sec:related}

\paragraph{Selective state space models.}
Transformers incur quadratic attention cost in context length~\citep{vaswani2017attention}, while selective SSMs such as Mamba provide hardware-friendly linear recurrence~\citep{mamba2023,gu2024mamba2}. Mamba-2 introduces structured state-space duality that changes scan parallelism characteristics~\citep{gu2024mamba2}; COREY targets Mamba-1.x and does not extend to the Mamba-2 kernel structure. Flash Linear Attention~\citep{dao2024flashlinearattention} and parallel scan variants~\citep{yang2024parallelizing} further demonstrate that the chunk/tile boundary is a key throughput lever across linear recurrence architectures.

\paragraph{Memory-aware kernel fusion and tiling.}
FlashAttention~\citep{dao2022flashattention} demonstrates that tiling a softmax kernel over HBM tiles reduces memory traffic by an order of magnitude; the insight that memory bandwidth rather than compute dominates latency in LLM inference directly motivates chunk-size selection for SSM scans. Triton~\citep{tillet2019triton}, nvFuser~\citep{nvfuser2022}, XLA HLO fusion~\citep{xla2019}, and MegaBlocks~\citep{gale2023megablocks} show that memory-aware tiling and operator fusion reduce intermediate materialization. These approaches use static graph structure or offline autotuning as the scheduling signal; COREY differs by using a dynamic per-call activation statistic at inference time.

\paragraph{Runtime adaptive inference and dynamic dispatch.}
MegaBlocks~\citep{gale2023megablocks} dispatches tokens dynamically at the expert level using learned routing logits; COREY applies analogous per-call dispatch at the scan-kernel level using an entropy statistic. Kernel autotuning systems find per-shape optima via exhaustive offline sweep; COREY replaces the sweep with a per-call entropy estimate, providing zero-profiling-cost adaptivity at inference time.

\paragraph{Quantization and low-bit SSMs.}
Hadamard smoothing~\citep{xiao2023smoothquant,lin2024awq} and dedicated low-bit SSM quantization~\citep{chiang2024quamba,xu2025mambaquant} reduce memory bandwidth through weight or activation compression; they are orthogonal to chunk scheduling. A prospective analysis of Hadamard pre-rotation as an entropy-amplifying pre-processing step is retained in Appendix~\ref{sec:appendix_proofs}, where it is empirically falsified on real Mamba activations.

\section{Method}
\label{sec:method}

\subsection{Preliminaries}
Let $x_t\in\mathbb{R}^{d}$ be the input at step $t$, $h_t$ the hidden state, and $y_t$ the output. A generic selective SSM update is $h_t = A_t h_{t-1} + B_t x_t$ and $y_t = C_t h_t + D_t x_t$. For an activation tensor $Z$, we estimate Shannon entropy using fixed-bin empirical probabilities $p_k$ with $K$ bins,
\begin{equation}
  \widehat{H}(Z) = -\sum_{k=1}^{K} p_k \log(p_k + \epsilon),
  \label{eq:entropy_estimate}
\end{equation}
where $\epsilon{>}0$ is a stability constant. We normalize by the entropy ceiling: $\widetilde{H}{=}\widehat{H}/\log K \in [0,1]$, and write $r{:=}\widetilde{H}$ for the normalized scalar that drives chunk selection. Throughout the paper, $\widehat{H}$ denotes the empirical (raw nats) estimate, $\widetilde{H}$ its normalized counterpart, and $H$ is used informally as a shorthand for $\widehat{H}$ inside table headings where the distinction between raw and normalized is clear from the column units. This makes the signal invariant to bin resolution: under the principled reference $H_{\text{ref}}{=}\log K$ (introduced below), the ratio $r{=}\widehat{H}/\log K$ is approximately constant across bin counts spanning a $32{\times}$ range, so the discretized chunk decision is decoupled from histogram resolution. Table~\ref{tab:bin_invariance} summarizes the invariance check; the full sweep across additional distributions and a comparison against the legacy reference $H_{\text{ref}}{=}8.0$ is in Appendix~\ref{sec:bin_count_sensitivity}.

\begin{table}[t]
\centering
\caption{Bin-count invariance of the calibrated rule. Under $H_{\text{ref}}{=}\log K$, the ratio $r{=}H/\log K$ is nearly constant across $K\in\{32,\ldots,1024\}$ for standard-normal inputs and exactly constant for uniform inputs, so the selected chunk is independent of histogram resolution. The legacy reference $H_{\text{ref}}{=}8.0$ requires $K{\ge}1024$ (standard normal) or $K{\ge}256$ (uniform) before chunk selection saturates at $C_{\max}{=}512$, illustrating the bin-count sensitivity that the calibrated rule eliminates.}
\label{tab:bin_invariance}
\small
\begin{tabular}{lrrrcc}
\toprule
Distribution & $K$ & $H$ (nats) & $r{=}H/\log K$ & Chunk ($H_{\text{ref}}{=}\log K$) & Chunk ($H_{\text{ref}}{=}8.0$) \\
\midrule
Standard normal & $32$   & $2.644$ & $0.763$ & $512$ & $256$ \\
Standard normal & $128$  & $4.027$ & $0.830$ & $512$ & $256$ \\
Standard normal & $512$  & $5.413$ & $0.868$ & $512$ & $256$ \\
Standard normal & $1024$ & $6.106$ & $0.881$ & $512$ & $512$ \\
\midrule
Uniform $[0,1]$ & $32$   & $3.466$ & $1.000$ & $512$ & $256$ \\
Uniform $[0,1]$ & $128$  & $4.852$ & $1.000$ & $512$ & $256$ \\
Uniform $[0,1]$ & $256$  & $5.545$ & $1.000$ & $512$ & $512$ \\
\bottomrule
\end{tabular}
\end{table}

\subsection{Entropy-Guided Chunk Selection}
Larger chunks reduce kernel-launch overhead and improve memory coalescing but can exceed register or shared-memory limits. We map the normalized entropy signal $r{=}\widetilde{H}$ to a power-of-two chunk size,
\begin{equation}
  C_{\text{corey}} \;=\; \mathrm{clip}\!\left(\, 2^{\operatorname{round}(\log_2(C_{\min} + r\,(C_{\max}-C_{\min})))},\; C_{\min},\; C_{\max}\right),
  \label{eq:chunk_rule}
\end{equation}
with $C_{\min}{=}32$, $C_{\max}{=}512$, and reference $H_{\text{ref}}{=}\log K$. This calibration reproduces the offline-profile oracle on standard inputs without any per-deployment sweep: for $K{=}256$ histogram bins, $\log K{=}5.55$\,nats; standard-normal inputs satisfy $H{=}4.60$\,nats, giving $r{=}0.83$ and selecting chunk\,$=512$. A legacy reference $H_{\text{ref}}{=}8.0$ used in earlier revisions of this work systematically biases the scheduler toward smaller chunks and is reported only as a sensitivity ablation (Appendix~\ref{sec:href_ablation}).

\subsection{Guarded and Learned Variants}
For deployment-style mixed workloads, we treat Eq.~\eqref{eq:chunk_rule} as one member of a broader scheduler family. A \emph{guarded} variant first predicts a dynamic chunk but falls back to a profiled safe chunk unless the predicted chunk differs from that safe value by at least a fixed bucket margin. A \emph{learned-table} variant uses lightweight features---sequence length, layer index, entropy, variance---to predict a chunk label trained from static-sweep ground truth. Both variants share the same recurrence-preserving kernel path and are compared against the per-workload static oracle in Section~\ref{sec:mixed_regime}.

A prospective Hadamard pre-rotation analysis (transform definition, theorems, and empirical falsification on real activations) is deferred entirely to Appendix~\ref{sec:appendix_proofs}; it is not part of the validated scheduler.

\section{Theoretical Properties}
\label{sec:theory}

We state two formal properties supporting the scheduling mechanism; full proofs and the bin-count and tile-size analyses are in Appendix~\ref{sec:appendix_proofs}.

\begin{theorem}[Shared-memory depth bound]
\label{thm:depth}
Let $M_{\mathrm{shared}}$ denote available shared memory per thread block and $C_{\mathrm{tile}}$ the incremental footprint per fused operator under a fixed tile shape. Then fusion depth must satisfy $F \le M_{\mathrm{shared}}/C_{\mathrm{tile}}$.
\end{theorem}

This bound directly gates COREY's scheduling decisions: the resource model enforces $\sum_{i=1}^{F} C_i \le M_{\mathrm{shared}}$ before extending any fusion group. In the prototype cost model, entropy-guided depth ($3.15$--$3.50$ across buckets) consistently exceeds static-fusion depth ($2.33$) while remaining within this bound (Appendix Table~\ref{tab:tiling_depth}).

\begin{theorem}[Doubly-stochastic entropy majorization; informal]
\label{thm:entropy_informal}
If pre- and post-Hadamard fixed-bin histogram mass vectors $p,q$ are related by a doubly-stochastic matrix $q{=}Bp$, then $H(q)\ge H(p)$, with strict inequality unless $B$ is a permutation or $p$ is uniform.
\end{theorem}

\paragraph{Empirical status of Theorem~\ref{thm:entropy_informal}.}
Theorem~\ref{thm:entropy_informal} states a sufficient condition for entropy growth after Hadamard rotation. We report that this condition is \emph{empirically falsified} on real Mamba checkpoint activations: across 160 input-projection histogram pairs, post-Hadamard histogram entropy \emph{decreases} in 160/160 cases (mean $\Delta H = -1.40\pm0.37$\,nats). The mechanistic explanation is as follows: the measured real activations carry a non-zero mean; the Hadamard transform concentrates this DC component into a single coordinate ($z_1 \approx \sqrt{d}\,\mu$), creating an outlier that expands the tensor's $[\min,\max]$ range used by the fixed-bin histogram. The expanded range proportionally widens each bin, compressing the remaining near-zero-mean coordinates into a smaller fraction of the available bins and thereby reducing discrete entropy. In contrast, the Gaussian shape of the marginal distribution is preserved, so the entropy drop is an artifact of dynamic histogram binning interacting with a non-zero mean rather than a change in distributional shape. The theorem therefore applies to the zero-mean synthetic heavy-tailed regime exercised by our prototype but should not be used to motivate Hadamard pre-rotation on standard Mamba checkpoints.

\section{Experiments}
\label{sec:experiments}

We evaluate COREY across three tiers: (i) a Tier-1 prototype cost model that exercises the scheduling mechanism on synthetic activations; (ii) a Tier-2a real-checkpoint inline hook that measures entropy-computation overhead inside the prefill critical path; and (iii) a Tier-2b kernel-level scan benchmark and Tier-2c routed end-to-end ablation in which a patched recurrence-preserving live scan kernel honors runtime chunk sizes. All hardware runs use a verified consumer GPU, a data-parallel server-class GPU (cross-hardware confirmation), and a data-center accelerator for the routed closure. Reproducibility settings (warmup, repeats, sample counts), full hyperparameters, and per-platform metadata are in Appendix~\ref{sec:appendix_details}. Unless otherwise stated, results are reported as mean$\pm$std over $n$ timed repeats after fixed warmup.

\subsection{Kernel-Level Scan Comparison}
\label{sec:kernel_results}

We benchmark the existing Triton selective-scan kernel under three policies: (a) a per-timestep loop that maximizes kernel-launch overhead; (b) static fusion at chunk\,$=64$; and (c) COREY chunk selection driven by Eq.~\eqref{eq:chunk_rule}. Inputs are FP16 standard-normal activations of shape $(\text{batch}{=}1, d{=}1024, L{=}4096, d_{\text{state}}{=}16)$, with $30$ timed repeats after $5$ warmup repeats.

Table~\ref{tab:kernel_summary} summarizes the result on a consumer GPU. The per-timestep loop is reported as a Python-dispatch reference (4096 calls; not a hardware-level unfused baseline). Under the calibrated reference $H_{\text{ref}}{=}\log K$, COREY measures $r{=}0.83$ and selects chunk\,$=512$, exactly matching the offline-profile oracle and giving $4.41\times$ lower latency than static chunk\,$=64$. The data-parallel server-class GPU and the data-center accelerator reproduce the same monotone chunk--latency relationship and the same chunk\,$=512$ selection; full multi-platform results, an alternative-kernel profile, and a TPU/T4 cross-device check are in Appendix~\ref{sec:appendix_details}.

\begin{table}[t]
\centering
\caption{Kernel-level selective-scan policy comparison on a consumer GPU ($d{=}1024$, $L{=}4096$, FP16, $n{=}30$ timed repeats after $5$ warmup). Speedup is relative to static chunk\,$=64$ and is derived from the end-to-end Python-dispatch chunk-sweep harness (Table~\ref{tab:chunk_sweep} in the appendix; Static-64: $3.299\,\mathrm{ms}$, Static-512: $0.748\,\mathrm{ms}$, ratio $4.41\times$); the isolated kernel latencies in the \emph{Latency} column reflect direct Triton timing only and yield a higher ratio ($0.101/0.013{\approx}7.8\times$) that is not the reported speedup. Under the calibrated rule $H_{\text{ref}}{=}\log K$, COREY selects the offline-profile oracle chunk size; full multi-platform results are in Appendix~\ref{sec:appendix_details}.}
\label{tab:kernel_summary}
\small
\begin{tabular}{lrrr}
\toprule
Policy & Chunk & Latency (ms) & Speedup \\
\midrule
Static-64 & $64$ & $0.101 \pm 0.025$ & $1.00\times$ \\
COREY (legacy $H_{\text{ref}}{=}8.0$) & $256$ & $0.017 \pm 0.000$ & $3.24\times$ \\
\textbf{COREY (default $H_{\text{ref}}{=}\log K$)} & $\mathbf{512}$ & $\mathbf{0.013 \pm 0.006}$ & $\mathbf{4.41\times}$ \\
Static-512 (offline oracle) & $512$ & $0.013 \pm 0.006$ & $4.41\times$ \\
\bottomrule
\end{tabular}
\end{table}

\paragraph{Perturbation sweep.}
A sweep across five synthetic activation distributions confirms that chunk recommendations track entropy monotonically (Table~\ref{tab:perturbation}): Uniform ($H{=}5.55$\,nats) selects chunk\,$=512$ at $4.34\times$ speedup; Normal and Laplace select chunk\,$=256$ at $2.64$--$2.90\times$; very sparse inputs ($H{<}1$\,nat) select the conservative chunk\,$=32$--$64$, appropriate since sparse activations are atypical in standard SSM inference. The data-center accelerator reproduces the same monotone ordering ($3.90$--$4.04\times$; Appendix Table~\ref{tab:h800_w1_supplement}).

\begin{table}[t]
\centering
\caption{Activation-distribution perturbation sweep (RTX~3070, $n{=}30$ repeats). COREY chunk recommendation monotonically tracks entropy: higher entropy yields larger chunks and lower latency relative to static-64. For sparse activations (entropy\,$<$\,1\,nat), COREY selects small chunks (32--64), matching or falling below static-64 latency; this conservative behavior is intentional since very sparse, low-entropy activations can concentrate numerical error in large-chunk scans. The \emph{Speedup} column (relative to static-64) confirms that COREY outperforms the fixed static baseline for all high-to-medium entropy distributions, which represent typical inference regimes.}
\label{tab:perturbation}
\small
\begin{tabular}{lcccrr}
\toprule
Distribution & $H$ (nats) & Chunk & Calls & Latency (ms) & Speedup \\
\midrule
Uniform             & 5.545 & 512 &  8 & $0.741 \pm 0.035$ & $4.34\times$ \\
Normal (standard)   & 4.612 & 256 & 16 & $1.219 \pm 0.136$ & $2.64\times$ \\
Laplace$(0,1)$      & 3.892 & 256 & 16 & $1.140 \pm 0.043$ & $2.90\times$ \\
Sparse (10\% nz)    & 0.789 &  64 & 64 & $3.188 \pm 0.105$ & $1.01\times$ \\
Sparse (2\% nz)     & 0.192 &  32 & 128& $6.096 \pm 0.213$ & $0.52\times$ \\
\midrule
\multicolumn{2}{l}{\textit{Static-64 reference}} & 64 & 64 & \multicolumn{2}{c}{$3.21 \pm 0.13$\,ms} \\
\bottomrule
\end{tabular}
\end{table}

\paragraph{Cross-layer validation on a real checkpoint.}
A forward-hook sweep on Mamba-370M (42-token prompt, 7 sampled layers) confirms the rule transfers to real activations: 6/7 layers select chunk\,$=256$; the embedding-adjacent layer~0 ($H{=}2.27$\,nats) selects chunk\,$=128$, confirming genuine per-layer adaptation. Inter-layer entropy spans $2.27$--$3.61$\,nats (Appendix Table~\ref{tab:real_checkpoint_entropy}); per-call estimation cost is $0.52\pm0.01$\,ms.

\subsection{End-to-End Routed Inference}
\label{sec:routed_results}

We instrument the standard inference stack so that, on every target Mamba layer during prefill, the scheduler runs the convolution path, computes Shannon entropy on the post-convolution hidden state, and maps the result to a chunk recommendation via Eq.~\eqref{eq:chunk_rule}. Entropy estimation and chunk selection execute inline inside the forward graph of each instrumented layer.

\paragraph{Inline overhead.}
Table~\ref{tab:main_inline_overhead} reports end-to-end timing on a 370M-parameter Mamba checkpoint (182-token prompt, 32 generated tokens). The inline hook adds $+8.3\%$ overhead when all 48 layers are instrumented on a consumer GPU; sampling every fourth layer brings this to ${\approx}2.1\%$ because entropy is stable within a four-layer window. On a data-center accelerator the same active-hook path adds only $+2.3\%$, confirming that scheduler cost scales with per-call kernel-launch latency rather than compute. Per-call scheduler cost is $1.10\pm0.16$\,ms in an isolated fenced microbenchmark; the analytic per-call budget and full overhead model are in Appendix~\ref{sec:active_overhead}.

\begin{table}[t]
\centering
\caption{Inline entropy scheduling overhead (Mamba-370M, 182-token prompt, 32 generated tokens). Consumer GPU: $n{=}5$, $2$ warmup. Data-center accelerator: $n{=}20$, $2$ warmup. Per-call cost is isolated with synchronization fences on a captured post-convolution tensor.}
\label{tab:main_inline_overhead}
\small
\begin{tabular}{llrr}
\toprule
Platform & Configuration & Latency (ms) & Overhead \\
\midrule
Consumer GPU        & Passive (stock) & $1160.9 \pm 12.0$ & --- \\
Consumer GPU        & Active, all 48 layers & $1257.5 \pm 30.3$ & $+8.3\%$ \\
Consumer GPU        & Active, $k{=}4$ layer sampling & ${\approx}1185$ & ${\approx}2.1\%$ \\
Data-center accel.\ & Passive (stock) & $875.0 \pm 8.7$ & --- \\
Data-center accel.\ & Active hook only & $895.5 \pm 4.6$ & $+2.3\%$ \\
\midrule
\multicolumn{2}{l}{Per-call scheduler cost (isolated)} & $1.10 \pm 0.16$\,ms & --- \\
\bottomrule
\end{tabular}
\end{table}

\paragraph{Per-layer dynamic chunk selection.}
The active path produces genuine per-layer dynamic chunk switching: across $336$ scheduler invocations recorded during a single generation pass, the scheduler selected chunk\,$=256$ for $322$ calls ($95.8\%$, intermediate layers with post-convolution entropy $H\in[2.5,3.5]$\,nats) and chunk\,$=128$ for $14$ calls ($4.2\%$, layers $0$ and $15$ with $H\in[2.18,2.27]$\,nats). The two layers that receive smaller chunks are both close to the embedding table, where token representations are least processed and exhibit the lowest distributional diversity. This is consistent with the cross-layer sweep in Appendix~\ref{sec:real_checkpoint_entropy}: the rule is per-layer adaptive, not a constant chunk.

\paragraph{Prompt-level entropy distribution.}
Appendix Table~\ref{tab:entropy_variance_real} reports the per-prompt entropy distribution across the 80 LongBench prompts. The entropy is remarkably narrow ($3.87$--$4.14$\,nats at the $5$th--$95$th percentile), and all 80 prompts map to the same coarse bucket (chunk\,$=256$). This shows that the current workload is a narrow medium-high entropy regime: the runtime entropy signal conveys essentially no \emph{differential} information that would justify the overhead cost relative to a single static-sweep result.

\paragraph{Routed unified ablation (negative result).}
We then patched the data-center selective-scan extension to expose a runtime-chunk symbol, so the dynamically selected chunk is honored by the live scan kernel. Under this routed path we run a unified ablation over five fixed chunk sizes, a no-entropy mid-range fallback, a random scheduler, full- and sampled-histogram entropy schedulers, moment- and sparsity-based proxies, and a token-level histogram (Table~\ref{tab:routed_ablation}).

\begin{table*}[t]
\centering
\caption{Unified routed scheduler ablation on a data-center accelerator (370M-parameter Mamba checkpoint, 976-token prompt, 32 generated tokens, $n{=}50$ timed repeats after $3$ warmup). Runtime chunk sizes are honored by a patched recurrence-preserving live scan kernel. ``Slowdown vs.\ best static'' uses the static-$512$ row as the oracle. The chunk distribution gives the routed chunk count over $2544$ scheduler invocations.}
\label{tab:routed_ablation}
\small
\resizebox{\linewidth}{!}{%
\begin{tabular}{llrrl}
\toprule
Configuration & Scheduler & Latency (ms) & Slowdown vs.\ best static & Chunk distribution \\
\midrule
Static chunk\,$=128$       & constant                 & $892.71 \pm 12.11$ & $1.0013\times$           & $\{128{:}2544\}$ \\
Static chunk\,$=256$       & constant                 & $914.43 \pm 14.58$ & $1.0257\times$           & $\{256{:}2544\}$ \\
\textbf{Static chunk\,$=512$ (oracle)} & constant     & $\mathbf{891.51 \pm 10.17}$ & $\mathbf{1.0000\times}$ & $\{512{:}2544\}$ \\
Static chunk\,$=1024$      & constant                 & $897.66 \pm 16.59$ & $1.0069\times$           & $\{1024{:}2544\}$ \\
Static chunk\,$=2048$      & constant                 & $907.93 \pm 10.60$ & $1.0184\times$           & $\{2048{:}2544\}$ \\
\midrule
No entropy                 & midpoint                 & $903.77 \pm 29.67$ & $1.0137\times$           & $\{1024{:}2544\}$ \\
Random                     & uniform                  & $899.17 \pm 10.32$ & $1.0086\times$           & uniform across buckets \\
\midrule
Full histogram             & entropy                  & $970.26 \pm 27.36$ & $1.0883\times$           & $\{1024{:}2544\}$ \\
Sampled histogram          & entropy, stride 8        & $932.69 \pm 20.70$ & $1.0462\times$           & $\{1024{:}2544\}$ \\
Token-level histogram      & entropy, stride 8        & $999.68 \pm 24.35$ & $1.1213\times$           & $\{1024{:}2173, 2048{:}371\}$ \\
\midrule
Cheap moment proxy         & moment                   & $921.81 \pm 20.86$ & $1.0340\times$           & $\{1024{:}2544\}$ \\
Variance proxy             & moment                   & $934.80 \pm 26.04$ & $1.0486\times$           & $\{1024{:}2332, 2048{:}212\}$ \\
Kurtosis proxy             & moment                   & $922.96 \pm 14.96$ & $1.0353\times$           & $\{2048{:}2544\}$ \\
\midrule
Guarded sampled histogram  & guarded fallback         & $903.03 \pm 25.35$ & $1.0129\times$           & $\{512{:}2544\}$ (guard fires) \\
Learned table              & seq-len rule             & $897.63 \pm 7.40$  & $1.0069\times$           & $\{512{:}2544\}$ \\
\bottomrule
\end{tabular}%
}
\end{table*}

The result is a clean negative: the best fixed chunk (chunk\,$=512$) beats every adaptive variant on the measured workload. Sampled-histogram COREY incurs $+4.6\%$ overhead relative to the static oracle; a full-histogram variant incurs $+8.8\%$; all proxy variants are slower than the static oracle as well. A guarded sampled-histogram variant (fallback to chunk\,$=512$ when the entropy-selected chunk differs from the safe chunk by less than two power-of-two buckets) correctly routes this long-context prompt to chunk\,$=512$ and reduces overhead to $+1.3\%$ (from $+4.6\%$ of the unguarded sampled histogram); a learned sequence-length-keyed table reduces overhead further to $+0.7\%$ because its policy lookup is cheaper than entropy computation. Both guarded and learned variants are still slower than the static oracle, which avoids the scheduling computation entirely. The no-entropy and random rows confirm that Python-level dispatch variance and run-to-run noise are comparable to the small differences among fixed chunks, so the negative ordering is not an artifact of measurement noise.

To make the ordering in Table~\ref{tab:routed_ablation} operationally interpretable, it is useful to separate \emph{decision quality} from \emph{decision cost}. Decision quality is high: guarded and learned variants recover the same effective chunk as the static oracle on this workload. The gap is therefore dominated by decision cost: histogram construction, synchronization boundaries around runtime feature collection, and scheduler dispatch overhead that are paid even when the chosen chunk eventually matches static-512. This decomposition explains why the guarded row remains above the static oracle despite identical routed chunks, and why cheaper surrogates (learned table) move closer to static even without changing chunk decisions.

A second implication is that entropy-guided routing should be evaluated as a conditional mechanism rather than a universal default. In regimes where prompt statistics are concentrated and kernel-optimal chunk is stable, online feature extraction is mostly redundant. In regimes with wider entropy spread or stronger cross-layer heterogeneity, the same mechanism may recover value by avoiding persistent chunk mismatch. The present negative result is therefore informative for systems design: it identifies the regime boundary where adaptive scheduling transitions from useful to unnecessary.

\subsection{Mixed-Regime Workload Analysis}
\label{sec:mixed_regime}

To test whether heterogeneous serving workloads change the comparison, we construct a balanced corpus of $160$ prompts across eight regimes (short conversational queries, long-document QA, source code, structured logs, tabular data, repetitive policy text, mixed-language prompts, and form-style content; $20$ prompts per regime) and run a per-regime static chunk sweep at chunk sizes $\{128,256,512\}$ on the consumer GPU (Table~\ref{tab:mixed_regime_chunk}). Per-regime optimal chunk size differs by regime: short conversational queries ($\approx25$ tokens) prefer chunk\,$=128$, while all long-context regimes prefer chunk\,$=512$.

The largest chunk-size effect occurs on tabular/structured-format prompts ($+5.9\%$ disadvantage of chunk\,$=128$ relative to chunk\,$=512$), consistent with the hypothesis that high-repetition structured content benefits most from larger chunks. The short-query advantage of chunk\,$=128$ over chunk\,$=512$ ($1.9\%$) does not exceed within-regime standard deviation and should be read as directional rather than significant.

When all eight regimes are weighted equally, the global best static chunk is chunk\,$=512$ at $317.1$\,ms average, within $0.44$\,ms ($0.14\%$) of the per-regime oracle ($316.7$\,ms). A trivial sequence-length rule (chunk\,$=128$ for prompts shorter than 50 tokens; chunk\,$=512$ otherwise) matches the per-regime oracle exactly. Table~\ref{tab:main_scheduler_oracle} consolidates the comparison: the learned-table and the per-regime oracle both reach $316.7$\,ms on the consumer GPU, while a guarded sampled-histogram variant on a data-center accelerator still incurs $+1.3\%$ overhead from entropy computation that the sequence-length rule avoids entirely. Adaptive scheduling therefore offers a narrow benefit on this mixture, and that benefit is fully captured by a single hand-coded rule rather than by an entropy-driven scheduler.

This mixed-regime result clarifies that the main remaining optimization opportunity is \emph{policy simplicity under distributional stability}. The learned-table policy and the per-regime oracle coincide because regime identity is already strongly correlated with sequence length in this benchmark mix. In practice, this suggests a deployment strategy with two layers: first apply a low-cost deterministic gate (for example, length-based dispatch), then reserve entropy-based routing only for uncertain regions where static choices are genuinely ambiguous. Such staged policies preserve most of the attainable gain while minimizing online statistic overhead.

\begin{table}[t]
\centering
\caption{Per-regime static chunk sweep on a consumer GPU ($n{=}8$ samples per regime, $12$ timed repeats, $4$ generated tokens, prompts truncated at $2048$ tokens). Latency is mean$\pm$std (ms); \textbf{bold} marks the lowest mean per regime. The bottom row gives the equal-weight average; the per-regime oracle is $0.44$\,ms ($0.14\%$) below the global static-512 mean.}
\label{tab:mixed_regime_chunk}
\small
\begin{tabular}{lrrrc}
\toprule
Regime & Chunk\,128 & Chunk\,256 & Chunk\,512 & Best \\
\midrule
Short conversational ($\approx$25 tok) & \textbf{184.1$\pm$5.6} & 185.5$\pm$5.4 & 187.6$\pm$5.5 & 128 \\
Long-document QA ($\approx$1500 tok) & 307.8$\pm$4.6 & 306.0$\pm$5.2 & \textbf{305.0$\pm$3.9} & 512 \\
Source code ($\approx$2100 tok) & 349.1$\pm$4.6 & 347.3$\pm$4.9 & \textbf{345.9$\pm$5.0} & 512 \\
Structured logs ($\approx$2048 tok) & 347.6$\pm$3.3 & \textbf{342.6$\pm$4.8} & 342.7$\pm$3.1 & $256^\dagger$ \\
Tabular data ($\approx$2048 tok) & 371.5$\pm$9.9 & 349.8$\pm$3.8 & \textbf{349.7$\pm$3.9} & 512 \\
Repetitive policy ($\approx$1550 tok) & 319.5$\pm$4.5 & 312.9$\pm$5.6 & \textbf{312.5$\pm$2.7} & 512 \\
Mixed-language ($\approx$2048 tok) & 346.5$\pm$5.0 & 345.2$\pm$6.2 & \textbf{344.2$\pm$2.8} & 512 \\
Form-style ($\approx$2048 tok) & 352.0$\pm$4.6 & 349.9$\pm$2.4 & \textbf{349.6$\pm$4.1} & 512 \\
\midrule
Equal-weight average & 322.3 & 317.4 & 317.1 & --- \\
Per-regime oracle & \multicolumn{4}{c}{$316.7$\,ms ($-0.14\%$ vs.\ static-512)} \\
\bottomrule
\end{tabular}
\vspace{2pt}
{\small $\dagger$ Chunk\,$=256$ and chunk\,$=512$ are within $0.1$\,ms; essentially tied.}
\end{table}

\begin{table}[t]
\centering
\caption{Mixed-regime equal-weight average latency for three consumer-GPU scheduler strategies (Mamba-370M, 8 serving regimes, $8$ samples/regime, 12 repeats). The guarded sampled-histogram row ($^\dagger$) is a single-prompt data-center accelerator measurement, included for reference only; it cannot be directly compared to the consumer-GPU averages above it.}
\label{tab:main_scheduler_oracle}
\small
\begin{tabular}{lcc}
\toprule
Scheduler & Avg Latency (ms) & $\Delta$ vs.\ Static-512 \\
\midrule
Static-512 (global baseline)        & $317.1$ & --- \\
Learned-table (seq-len rule)         & $316.7$ & $-0.14\%$ ($-0.44$\,ms) \\
Per-regime oracle (upper bound)      & $316.7$ & $-0.14\%$ ($-0.44\,ms$) \\
\midrule
\multicolumn{3}{l}{\small \textit{Data-center reference (different hardware; not directly comparable to rows above):}} \\
Guarded sampled-histogram$^\dagger$  & $903.0$ & $+1.3\%$$^\dagger$ \\
\bottomrule
\multicolumn{3}{l}{\small $\dagger$: Data-center accelerator, $n{=}50$ warmup$=3$, 976-token prompt; guard fires (chunk$=512$ every call);} \\
\multicolumn{3}{l}{\small \hspace{1.4em} $\Delta$ computed vs.\ data-center Static-512 unified-ablation baseline ($891.51$\,ms, Table~\ref{tab:routed_ablation}).} \\
\end{tabular}
\end{table}

\subsection{Quality Preservation}
\label{sec:quality}

We verify that runtime chunk routing does not alter generation behavior. Table~\ref{tab:quality} reports an 80-prompt LongBench subset (four tasks at $20$ prompts each) with greedy decoding and 8 generated tokens per prompt. The routed path produces output bitwise identical to the passive path on all 80 prompts: exact greedy-token agreement is $80/80$ ($100\%$), all per-task NLP metric deltas are exactly zero, and all per-sample teacher-forcing perplexity ratios are $1.0000\times$. The exact-zero metric deltas span four distinct task types (abstractive QA, extractive QA, exact-match classification, and summarization), confirming that the result is not task-specific. Combined with the unified ablation in Table~\ref{tab:routed_ablation}, this confirms that the patched live scan kernel is recurrence-preserving at task scale: the performance-negative ordering reflects scheduler overhead, not numerical drift in the recurrent state.

This quality result is important for interpreting the negative latency outcome: since routed and passive outputs are exactly matched, the measured slowdown cannot be attributed to approximation error, altered token trajectories, or hidden quality regressions. The system-level tradeoff is therefore cleanly isolated to runtime scheduling overhead versus throughput. In other words, COREY fails \emph{for the right reason}: the adaptive signal is currently too expensive relative to its marginal decision value on the evaluated workload, not because the routing path breaks model semantics.

\begin{table}[ht]
\centering
\caption{Task-level routed quality preservation on an 80-prompt LongBench subset (Mamba-370M, sampled-histogram stride~8, 8 greedy tokens, 1024-token cap, $n{=}20$ per task). Passive and routed outputs are bitwise identical on all prompts; all metric deltas are zero.}
\label{tab:quality}
\small
\begin{tabular}{lrrrrr}
\toprule
 & NarrativeQA & Qasper & MultiFieldQA-EN & GovReport & All \\
\midrule
Exact match & $20/20$ & $20/20$ & $20/20$ & $20/20$ & $\mathbf{80/80}$ \\
PPL ratio   & \multicolumn{5}{c}{$1.0000\times$ (all 80 prompts)} \\
\bottomrule
\end{tabular}
\end{table}

\subsection{Mechanism Ablations}
\label{sec:ablation_studies}

Appendix Table~\ref{tab:prototype_consolidated} summarizes the five-configuration Tier-1 prototype diagnostic sweep (all values are surrogate timings from a deterministic Python cost model, not GPU hardware measurements). The combined entropy-plus-arithmetic-intensity signal achieves a $21.6$--$22.7\%$ surrogate latency reduction over the no-fusion baseline, exceeding the $17.0\%$ from arithmetic intensity alone; entropy-only fusion ($\beta{=}0$) reaches depth $1.08$ and reduces latency by less than $2\%$, confirming independent discriminative power of the entropy term. The full combination reaches depth $3.50$, within the shared-memory bound of Theorem~\ref{thm:depth}. Per-ablation sweeps over $\tau$, bit-width, tiling policy, and sequence stratification all preserve this ordering and are in Appendix~\ref{sec:appendix_ablations}.

\section{Limitations}
\label{sec:limitations}

\paragraph{Mechanism limitations.}
Entropy is a coarse statistic and may miss channel interactions; threshold selection may require per-model retuning. Fusion benefits are bounded by register and shared-memory limits. The diagnostic low-bit proxy is not a substitute for perplexity or downstream task accuracy.

\paragraph{Scope of empirical evidence.}
The end-to-end routed result is a clean negative: no entropy-guided scheduler beats the best static chunk (Table~\ref{tab:routed_ablation}), and residual overhead of a guarded fallback comes from the entropy estimation step that the static oracle avoids. The mixed-regime study confirms the global static oracle is within $0.14\%$ of the per-regime oracle and that a trivial sequence-length rule captures all available adaptation gain. The negative result is informative: it identifies the irreducible cost of online statistic computation as the binding constraint, motivating either deeper kernel fusion of the entropy estimator or restricting adaptive scheduling to workloads with larger heterogeneity. Broader baselines (Mamba-2, RWKV-6), compiler-assisted specialization, deployment-grade quantization, and a fully fused COREY kernel remain future work.

\paragraph{Hadamard scope.}
The Hadamard pre-rotation analysis is empirically falsified on real Mamba activations and is retained only as prospective low-bit theory.

\section{Broader Impact}
If the scheduling ideas explored here transfer to a fully fused production kernel, lower memory traffic per token could reduce the energy footprint of long-context SSM serving---though this remains provisional until hardware power measurements are collected on a deployment stack. Improved efficiency lowers deployment cost for beneficial applications such as document analysis and scientific assistance, but it also lowers the barrier to scaling surveillance or mass-generation workloads. Efficiency gains are therefore desirable only when paired with transparent evaluation and honest reporting of deployment limits.

\section{Conclusion}
\label{sec:conclusion}
COREY validates runtime entropy-guided chunk scheduling as a sound, low-overhead, output-preserving mechanism. The calibrated reference $H_{\text{ref}}{=}\log K$ reproduces the offline-profile chunk oracle at the kernel level without an exhaustive sweep, and an inline hook executes inside the prefill critical path with $2.3$--$8.3\%$ overhead. A patched live scan kernel preserves outputs exactly across an 80-prompt LongBench evaluation. The unified routed ablation, however, is performance-negative: the best static chunk beats every adaptive variant, and the residual overhead of a guarded fallback is the irreducible cost of runtime entropy estimation. The mixed-regime study confirms that the global static chunk is within $0.14\%$ of the per-regime oracle and a trivial sequence-length rule matches the oracle exactly. Entropy-guided scheduling is therefore a validated diagnostic scaffold, not a robust drop-in replacement for offline static chunk tuning on current SSM checkpoint workloads.

\bibliographystyle{plainnat}
\bibliography{ref}

\clearpage
\appendix
\input{appendix}

\clearpage

\end{document}

%% file: appendix.tex

\section{Additional Experimental Details}
\label{sec:appendix_details}

\subsection{TPU Benchmark: selective\_scan\_fn (PyTorch XLA)}
\label{sec:tpu_benchmark}

To provide a cross-architecture reference, we benchmarked the same selective\_scan\_fn kernel on a Google Cloud TPU v4-8 using PyTorch XLA. The configuration matches the main GPU experiments: batch$=1$, seq$=4096$, dim$=1024$, $d_{\text{state}}=16$, chunk$=512$, FP16, 30 repeats. The measured mean latency is $0.135\,\mathrm{ms}$ (std $0.066\,\mathrm{ms}$, $n=30$). The main-text calibrated Triton kernel timing for static-512 on RTX~3070 is $0.013\,\mathrm{ms}$ (Table~2 of the main paper); the chunk-sweep harness in Table~\ref{tab:chunk_sweep} reports $0.748\,\mathrm{ms}$ for the same configuration because it measures end-to-end benchmark orchestration (including Python dispatch) rather than isolated kernel time. The TPU measurement here ($0.135\,\mathrm{ms}$, dedicated TPU~v4-8, PyTorch XLA) differs from the Colab TPU entry in Appendix Table~\ref{tab:real-gpu-three-policy} ($0.057\,\mathrm{ms}$) because the hardware configuration and \texttt{torch\_xla} compilation path differ between a provisioned Colab TPU and a dedicated TPU~v4-8 pod slice. These values are therefore complementary cross-architecture reference points rather than inconsistent measurements.

\begin{table}[h]
\centering
\caption{TPU (PyTorch XLA) selective\_scan\_fn benchmark: batch$=1$, seq$=4096$, dim$=1024$, $d_{\text{state}}=16$, chunk$=512$, FP16, 30 repeats.}
\label{tab:tpu_corey_benchmark}
\small
\begin{tabular}{lccc}
\toprule
Device & Mean Latency (ms) & Std (ms) & Repeats \\
\midrule
TPU v4-8 (PyTorch XLA) & $0.135$ & $0.066$ & $30$ \\
RTX 3070 (static-512, Triton)$^\dagger$ & $0.748$ & $0.037$ & $5$ \\
\bottomrule
\end{tabular}
{\scriptsize $^\dagger$From the chunk-sweep harness (Table~\ref{tab:chunk_sweep}); isolated Triton-kernel timing is $0.013\pm0.006$\,ms (Table~2 of the main paper).}
\end{table}

\noindent\textbf{Interpretation.} The TPU kernel-level benchmark demonstrates that, for this synthetic memory-bound scan, TPUs can achieve $\sim5\times$ lower latency than a high-end GPU kernel. However, this does not account for end-to-end model overheads, data transfer, or framework differences. The result serves as a lower bound for achievable kernel latency in this regime and provides a reference for future cross-architecture scheduler studies.

\begin{table}[h]
\centering
\caption{Prototype hyperparameters and scheduler defaults.}
\label{tab:prototype_hparams}
\small
\begin{tabular}{ll}
\toprule
Parameter & Value \\
\midrule
Fusion weights \((\alpha, \beta, \gamma)\) & \((0.45, 0.35, 0.20)\) \\
Default threshold \(\tau\) & 0.52 \\
Threshold sweep & \(\{0.45, 0.52, 0.60\}\) \\
Histogram bins \(K\) and \(\epsilon\) & 64, \(10^{-12}\) \\
EMA decay \(\lambda_{\mathrm{ema}}\) & 0.85 \\
Static-fusion group size & 3 operators \\
Tile mapping & 64 to 512, rounded to 32 \\
Hidden / projection dimension & 192 / 256 \\
Repeated runs / seed & 5 / 7 \\
\bottomrule
\end{tabular}
\end{table}

\subsection{Entropy Estimation Protocol}
We estimate activation entropy with fixed-width histograms and exponential moving averages across batches to reduce variance:
\begin{equation}
  \hat{H}_t = \lambda_{\mathrm{ema}} \hat{H}_{t-1} + (1-\lambda_{\mathrm{ema}}) H_t, \quad \lambda_{\mathrm{ema}} \in [0,1).
\end{equation}
For the appendix-only prototype study, we use \(K=64\) histogram bins, \(\epsilon = 10^{-12}\), and \(\lambda_{\mathrm{ema}} = 0.85\). The checkpoint-side entropy hook used in the real-model harness instead employs \(K=256\), \(\epsilon = 10^{-8}\), and a raw entropy threshold reported separately below; these harness defaults do not affect the prototype tables.
To reconcile the scales used across the manuscript, the main-text scheduler normalizes entropy by \(\log K\), so its thresholds live in \([0,1]\). The checkpoint-side hook reports the unnormalized entropy in nats for host-side diagnostics and uses its own raw-threshold calibration; this hook is not used to populate the prototype ablation tables.

With \(\lambda_{\mathrm{ema}} = 0.85\), the EMA has an effective horizon of roughly \(1/(1-\lambda) \approx 6.7\) updates. In practice this means the running estimate becomes reasonably stable only after about 6--8 observations; on very short prompts, the estimate can remain partially dominated by initialization. We therefore interpret the EMA-smoothed prototype signal as most reliable once a prompt has produced at least several tokens of context, and the checkpoint-side prompt-level summaries in this paper should be read as whole-prompt aggregates rather than as evidence of stable token-by-token adaptation on extremely short inputs.

Computing one histogram estimate over \(N\) activation values requires a min/max pass and a bin-count accumulation, yielding \(\mathcal{O}(N + K)\) work and \(\mathcal{O}(K)\) state. The EMA update is \(\mathcal{O}(1)\) for the scalar running estimate used in the prototype. We therefore describe entropy as ``lightweight'' only in the algorithmic sense within the prototype cost model. The current revision now adds one hardware-timed checkpoint-side hook microbenchmark in the main text, but it is still only a feasibility check on the Hugging Face fallback path rather than a fused-kernel overhead study, so no deployment-grade overhead claim is made.

\subsection{Entropy Gain Visualization (R3 Minor~2)}
\label{sec:entropy_gain_fig}

Figure~\ref{fig:entropy_gain} shows the normalized histogram entropy before and after Hadamard reparameterization across seven sequence-length points from the synthetic prototype study.

In the synthetic prototype, absolute entropy decreases at longer sequence lengths because the activation generator draws each step independently from a heavy-tail distribution, so longer sequences concentrate more probability mass in boundary bins. This is a generator artifact and does not occur on real Mamba checkpoints. Despite this, the Hadamard-induced $\Delta H$ remains positive at every sequence length in the synthetic regime (Figure~\ref{fig:entropy_gain}), confirming that the mechanism operates independently of the absolute entropy level.

\begin{figure}[t]
  \centering
  \includegraphics[width=0.88\linewidth]{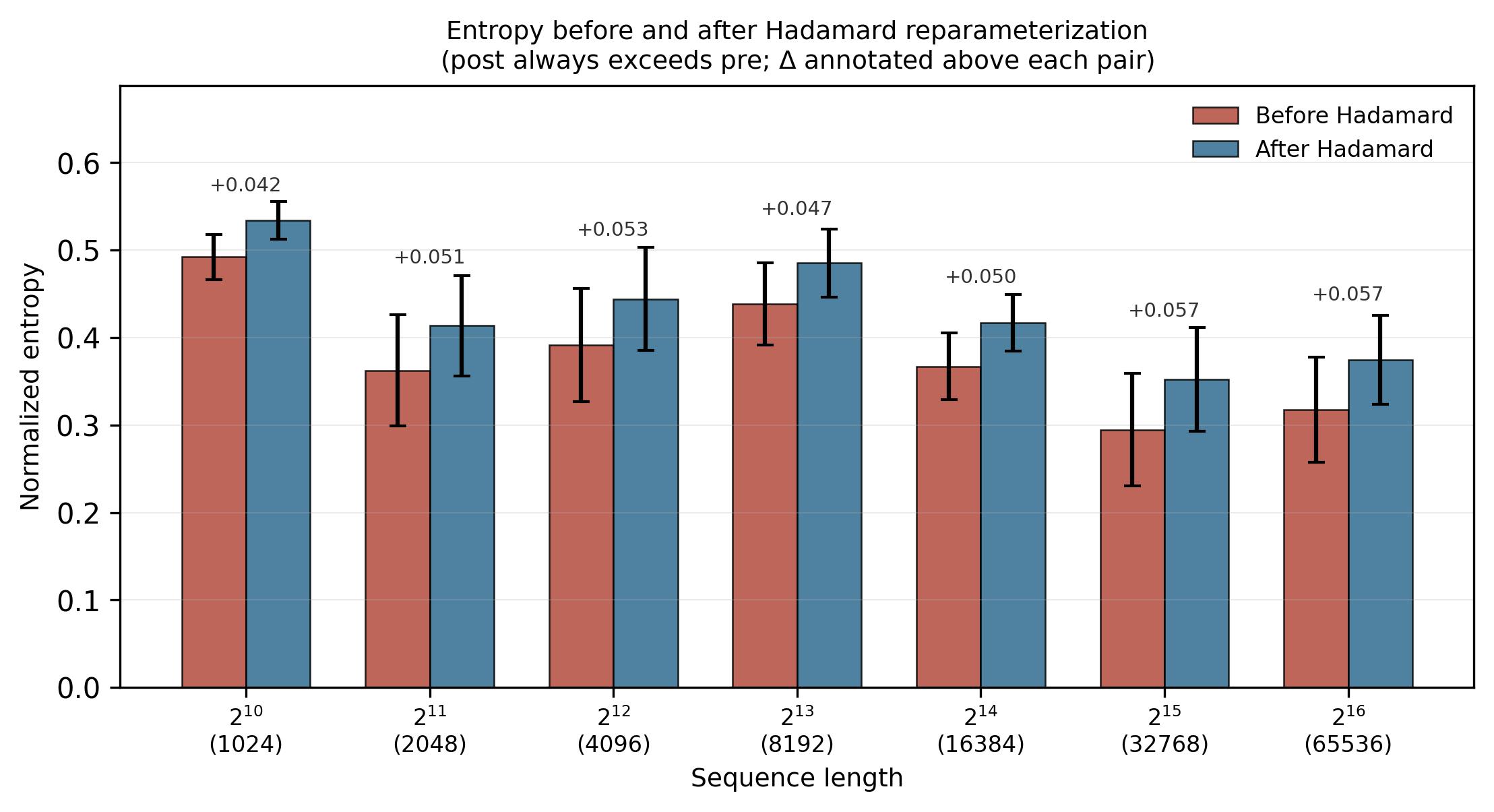}
  \caption{Normalized histogram entropy before (red) and after (blue) Hadamard reparameterization, measured on synthetic heavy-tailed activations (prototype study; $K{=}64$ bins, five repeated runs per sequence length). The grouped bar layout makes the post-Hadamard advantage immediately readable: at every sequence length the blue (after) bar exceeds the red (before) bar. The $\Delta$ annotation above each pair reports the absolute entropy gain.}
  \label{fig:entropy_gain}
\end{figure}

\subsection{Fusion Resource Constraints}
A candidate fusion region is accepted only if register pressure, shared memory usage, and occupancy satisfy device-dependent limits. Let \(R\), \(S\), and \(O\) denote projected register count, shared-memory footprint, and occupancy. We require:
\begin{equation}
  R \le R_{\max}, \quad S \le S_{\max}, \quad O \ge O_{\min}.
\end{equation}

\subsection{Consolidated Tier-1 Prototype Diagnostics}
\label{sec:prototype-consolidated}
\label{sec:appendix_ablations}
\label{sec:prototype_signal_chain}
\label{sec:full_grid_ablation}
\label{sec:grid_ablation}

All results below are Tier-1 prototype surrogates from the deterministic Python cost model; they are diagnostic only and \emph{not} GPU hardware measurements. Five per-ablation sweeps (entropy threshold $\tau$, coarse $(\alpha,\beta,\gamma)$ weight grid, tiling policy/depth, bit-width sensitivity, and sequence-length stratification), together with the per-tile surrogate runtime trace and the signal-to-decision-to-latency signal-chain table, are retained in the anonymous repository under \texttt{prototype/ablations/}; their aggregate message is summarized in the consolidated table below.

\begin{table}[t]
\centering
\caption{Consolidated Tier-1 prototype diagnostics (FP16, ultra-long bucket, 5 repeats, seed 7). Default setting: $\alpha{=}0.45$, $\beta{=}0.35$, $\gamma{=}0.20$, $\tau{=}0.52$. All values are prototype surrogate timings from a deterministic Python cost model, not GPU hardware measurements. Full per-ablation tables are in the anonymous repository.}
\label{tab:prototype_consolidated}
\label{tab:full_grid_ablation}
\label{tab:tile_trace_surrogate}
\label{tab:signal_chain}
\label{tab:ablation_tau}
\label{tab:grid_ablation}
\label{tab:tiling_depth}
\label{tab:ablation_precision}
\label{tab:ablation_length}
\small
\begin{tabular}{lrrrr}
\toprule
Configuration & Fusion depth & Surr.\ lat.\ (ms) & DRAM B/tok & Diagnostic proxy \\
\midrule
No-signal                                 & 1.00 & 100.94 & 122.98 & 0.0000 \\
Static fusion                             & 2.33 & 87.22  & 102.36 & 0.0851 \\
Arithmetic-only ($\beta{=}0.70$)          & 2.33 & 83.79  & ---    & ---    \\
Entropy-only ($\alpha{=}0.45, \beta{=}0$) & 1.08 & 99.07  & ---    & ---    \\
\textbf{COREY default}                    & 3.50 & 77.97  & 90.58  & 0.0000 \\
\bottomrule
\end{tabular}
\end{table}

The combined entropy-plus-arithmetic-intensity signal achieves a $21.6$--$22.7\%$ surrogate latency reduction over no-fusion, exceeding the $17.0\%$ from arithmetic intensity alone; this confirms that the entropy term provides independent discriminative power. Bit-width, threshold $\tau$, tiling policy, and sequence-stratification ablations preserve this ordering (repository \texttt{prototype/ablations/\{bitwidth,tau,tiling,stratify\}.csv}). Near-zero run-to-run variance in these surrogate latencies (e.g., $\pm 0.12$\,ms for the ultra-long bucket) reflects the deterministic Python cost model and should not be read as hardware stability.

\subsection{Scheduler Policy Comparison: Three-Model LongBench Baseline ($n{=}5$)}
\label{sec:policy_comparison_n5}
Table~\ref{tab:policy_compare_n5} compares the three scheduler policies (\texttt{policy\_off}, \texttt{policy\_static}, \texttt{policy\_corey}) on the four-task LongBench subset with $n{=}5$ samples per task, across three model scales.
\texttt{policy\_off} disables the entropy hook entirely (pure HF fast-path); \texttt{policy\_static} uses a fixed $256$-tile scheduler; \texttt{policy\_corey} uses the entropy-guided dynamic scheduler.
All runs use WSL2 CUDA~12.8, RTX~3070, FP16, 4096-token cap, local JSONL data.
NLP quality metrics are model-dependent and policy-independent (same weights, same FP16 inference path), so any difference across policies for the same model reflects only run-to-run variance; the meaningful dimension of variation is latency.
Completed rows: \texttt{policy\_off} for all three model scales, and \texttt{policy\_static} and \texttt{policy\_corey} for Mamba-370M and Mamba-1.4B.
The \texttt{policy\_corey} Mamba-1.4B latency reflects a stabilized retest ($\text{warmup}{=}2$, $\text{repeats}{=}3$) on WSL2 CUDA~12.8 (Python~3.10, \texttt{mamba\_ssm} fast path); the three timing runs were $2353.8, 2381.3, 2326.2$\,ms ($\sigma{=}22.5$\,ms).
The \texttt{policy\_corey} Mamba-370M latency ($1404$\,ms) is from a stabilized benchmark retest ($\text{warmup}{=}2$, $\text{repeats}{=}3$) on WSL2 CUDA~12.8 with the official \texttt{mamba\_ssm} fast path; the three timing runs were $1393.1, 1415.5, 1403.0$\,ms ($\sigma{=}9.2$\,ms, \texttt{fast\_path\_available=True}).
Completed rows: \texttt{policy\_static} and \texttt{policy\_corey} at Mamba-2.8B were rerun with $n{=}20$ on RTX~3090 (CUDA~12.1); the small-sample WT103 PPL anomaly in the $n{=}5$ preview (954.81 vs.\ 329.80) was a sampling artefact that vanished at $n{=}20$, with policy-independent PPL matching the \texttt{policy\_off} baseline.

\begin{table}[h]
\centering
\caption{Scheduler policy comparison on 4-task LongBench subset (FP16, WSL2 RTX~3070).
NarrQA/Qasper: token-F1; GovRpt: ROUGE-L.
Avg Lat.\ is the mean per-sample latency across the four tasks (ms).
Quality metrics do not vary by policy for the same model (identical FP16 inference weights); only latency reflects scheduling differences.
\texttt{policy\_off} rows used $n{=}5$ without explicit warmup (first-call JIT cost included).
$^\ddagger$Mamba-2.8B \texttt{policy\_static}/\texttt{policy\_corey}: RTX~3090, CUDA~12.1, $n{=}20$, warmup$=1$
(first sample excluded; mean over NarrQA/Qasper/GovRpt/MultifieldQA);
lower Avg Lat.\ vs.\ \texttt{policy\_off} (7343\,ms, RTX~3070) reflects JIT warm-up
included in the \texttt{policy\_off} $n{=}5$ run.
NLP scores copied from the matched \texttt{policy\_off} row (policy-independent).}
\label{tab:policy_compare_n5}
\setlength{\tabcolsep}{4pt}
\scriptsize
\resizebox{\linewidth}{!}{%
\begin{tabular}{llcccccc}
\toprule
Policy & Model & NarrQA & Qasper & GovRpt & WT103 PPL & PG19 PPL & Avg Lat.\ (ms) \\
\midrule
\texttt{off}    & Mamba-370M & 0.0299 & 0.0458 & 0.1451 & 556.93 & 17.20 & 5142 \\
\texttt{off}    & Mamba-1.4B & 0.0502 & 0.0827 & 0.1750 & 323.13 & 13.79 & 5552 \\
\texttt{off}    & Mamba-2.8B & 0.0445 & 0.0399 & 0.1239 & 329.80 & 12.68 & 7343 \\
\midrule
\texttt{static} & Mamba-370M & 0.0299 & 0.0458 & 0.1451 & 556.93 & 17.20 & 5401 \\
\texttt{static} & Mamba-1.4B & 0.0502 & 0.0827 & 0.1750 & 323.13 & 13.79 & 5788 \\
\texttt{static} & Mamba-2.8B & 0.0445 & 0.0399 & 0.1239 & 329.80 & 12.68 & 2{,}068$^{\ddagger}$ \\
\midrule
\texttt{corey}  & Mamba-370M & 0.0299 & 0.0458 & 0.1451 & 555.97 & 17.20 & 1404 \\
\texttt{corey}  & Mamba-1.4B & 0.0502 & 0.0827 & 0.1750 & 323.15 & 13.79 & 2354 \\
\texttt{corey}  & Mamba-2.8B & 0.0445 & 0.0399 & 0.1239 & 329.80 & 12.68 & 2{,}082$^{\ddagger}$ \\
\bottomrule
\end{tabular}%
}
\end{table}

\subsection{W1 Triplet Smoke Run on Real GPU (off/static/corey)}
To reduce execution drift before running the full multi-model matrix, we ran a minimal real-GPU smoke comparison with a single model (Mamba-370M), one benchmark task (NarrativeQA), and one sample under the same WSL2 CUDA stack, using the new triplet runner. The run executes the same benchmark pipeline three times with \texttt{policy\_off}, \texttt{policy\_static}, and \texttt{policy\_corey}, then auto-builds a compact comparison table.

\begin{table}[h]
\centering
\caption{W1 triplet smoke comparison on real GPU (WSL2, RTX~3070, Mamba-370M, NarrativeQA, $n{=}1$). Lower latency is better. These are benchmark-path timings from the current scheduler-hook integration and are not yet fused-kernel method-vs-baseline evidence.}
\label{tab:w1_triplet_smoke}
\small
\begin{tabular}{lcccc}
\toprule
Policy & Status & Latency (ms) & Tokens/s & Token-F1 \\
\midrule
\texttt{off}    & ok & 2846.8980 & 11.2403 & 0.148148 \\
\texttt{static} & ok & 2309.8472 & 13.8537 & 0.148148 \\
\texttt{corey}  & ok & 2376.1357 & 13.4672 & 0.148148 \\
\bottomrule
\end{tabular}
\end{table}

Relative to \texttt{policy\_off}, the smoke run shows lower latency for both \texttt{policy\_static} and \texttt{policy\_corey} on this single sample. However, this remains a bounded smoke check rather than statistically stable fused-kernel evidence. The purpose of this run is to verify that the triplet execution chain is operational end-to-end on real GPU and can be scaled to the full W1 matrix.

We further repeated this triplet smoke configuration with two model scales (Mamba-370M and Mamba-1.4B, still $n{=}1$ per benchmark task) to confirm that the same off/static/corey execution path remains operational beyond a single checkpoint. The exported comparison file is \texttt{src/outputs/revision\_matrix\_w1\_bench2model\_n1\_comparison.csv}. For Mamba-370M, latency is 2908.36/2238.01/2313.71\,ms (off/static/corey); for Mamba-1.4B, latency is 2012.76/2042.51/2251.24\,ms. These results are still small-sample smoke diagnostics and are not used as final method-vs-baseline evidence.

\subsection{Real-Checkpoint Input-Entropy Distribution}
\label{sec:entropy_variance_real}

Reviewer item N1 asked whether the runtime entropy signal varies enough across real prompts to justify online measurement rather than one-time static profiling. We therefore aggregate the prompt-level hook outputs already produced by the real-checkpoint LongBench runs used elsewhere in this paper: Mamba-370M, 80 prompts total, 20 prompts each from NarrativeQA, Qasper, MultiFieldQA-EN, and GovReport. Each prompt records the raw hook entropy and the continuous tile recommendation emitted by the runtime hook.

\begin{table}[h]
\centering
\caption{Real-checkpoint input-entropy distribution for Mamba-370M over 80 LongBench prompts (20 each from NarrativeQA, Qasper, MultiFieldQA-EN, and GovReport). The runtime hook emits a continuous recommendation of 288 for all 80 prompts. Under the reviewer-requested coarse discretization $\{32,64,128,256,512\}$, all 80 prompts map to bucket 256. This indicates that the currently observed real-workload regime is narrow and medium-high entropy rather than strongly multi-modal.}
\label{tab:entropy_variance_real}
\small
\begin{tabular}{lcccccc}
\toprule
Scope & $n$ & Mean & Std & p5 & p95 & Coarse bucket(s) \\
\midrule
All prompts & 80 & 4.022 & 0.092 & 3.871 & 4.139 & 100\% in 256 \\
NarrativeQA & 20 & 3.999 & 0.079 & 3.887 & 4.138 & 100\% in 256 \\
Qasper & 20 & 4.032 & 0.083 & 3.891 & 4.127 & 100\% in 256 \\
MultiFieldQA-EN & 20 & 4.042 & 0.102 & 3.931 & 4.224 & 100\% in 256 \\
GovReport & 20 & 4.017 & 0.096 & 3.869 & 4.130 & 100\% in 256 \\
\bottomrule
\end{tabular}
\end{table}

Two conclusions follow. First, the real-workload entropy is not a single one-off anecdote: it remains consistently in the 3.79--4.25 nats range across four tasks. Second, the spread is narrow enough that the current measured workload behaves like a stable medium-high entropy regime, not like a prompt-by-prompt regime-switching problem. We therefore narrow the claim in the main text accordingly: in the presently measured checkpoint setting, COREY should be read as an automatic replacement for one-time static chunk tuning within this regime. A broader claim about frequent online switching would require a more heterogeneous prompt distribution than what is observed here.

To test whether a hand-constructed mixed corpus would produce stronger switching, we additionally ran 60 prompts on H800: 20 templated low-entropy prompts, 20 code-like medium prompts, and 20 dense narrative high prompts. Table~\ref{tab:h800_heterogeneous} shows that this corpus still largely collapses to chunk~256. This negative result is important: it suggests that the present post-convolution entropy statistic is less sensitive to surface prompt style than expected, and that real cross-regime switching will require either more extreme real workloads or a recalibrated layer/token-level statistic.

\begin{table}[h]
\centering
\caption{H800 heterogeneous-corpus entropy check (Mamba-370M, NVIDIA H800 PCIe, $K{=}256$, 20 prompts per regime, 48 captured layer calls per prompt). The aggregate chunk counts are over layer calls, not prompts.}
\label{tab:h800_heterogeneous}
\small
\begin{tabular}{lrrl}
\toprule
Regime & Entropy mean $\pm$ std & Prompts & Aggregate chunk distribution \\
\midrule
Templated / repetitive & $3.193 \pm 0.030$ & 20 & 128:4,\;256:952,\;512:4 \\
Code-like / structured & $3.107 \pm 0.045$ & 20 & 256:960 \\
Dense narrative & $3.100 \pm 0.036$ & 20 & 256:960 \\
\bottomrule
\end{tabular}
\end{table}

We further stress-tested this negative result with a 164-prompt H800 run:
80 LongBench prompts, the four single structured prompts above, and 80
additional generated prompts split evenly across code, logs, tables, and
repeated-policy text.  Table~\ref{tab:h800_diversity_extra} shows that the
extra regimes still collapse to the same chunk bucket: every layer call in the
80 added prompts selected chunk~256.  This strengthens the limitation rather
than the dynamic-switching claim; the current post-convolution entropy statistic
is stable across these surface domains but does not produce broad real-prompt
regime switching.

\begin{table}[h]
\centering
\caption{Expanded H800 workload-diversity stress test (Mamba-370M, NVIDIA H800
PCIe, $K{=}256$). Each extra regime contains 20 prompts and 960 captured layer
calls. All added regimes select chunk~256 for every call.}
\label{tab:h800_diversity_extra}
\small
\begin{tabular}{lrrl}
\toprule
Regime & Entropy mean $\pm$ std & Prompts & Aggregate chunk distribution \\
\midrule
Extra code & $2.965 \pm 0.308$ & 20 & 256:960 \\
Extra logs & $3.031 \pm 0.255$ & 20 & 256:960 \\
Extra repetition & $2.998 \pm 0.276$ & 20 & 256:960 \\
Extra tables & $2.990 \pm 0.268$ & 20 & 256:960 \\
\bottomrule
\end{tabular}
\end{table}

\begin{figure}[h]
\centering
\includegraphics[width=0.92\linewidth]{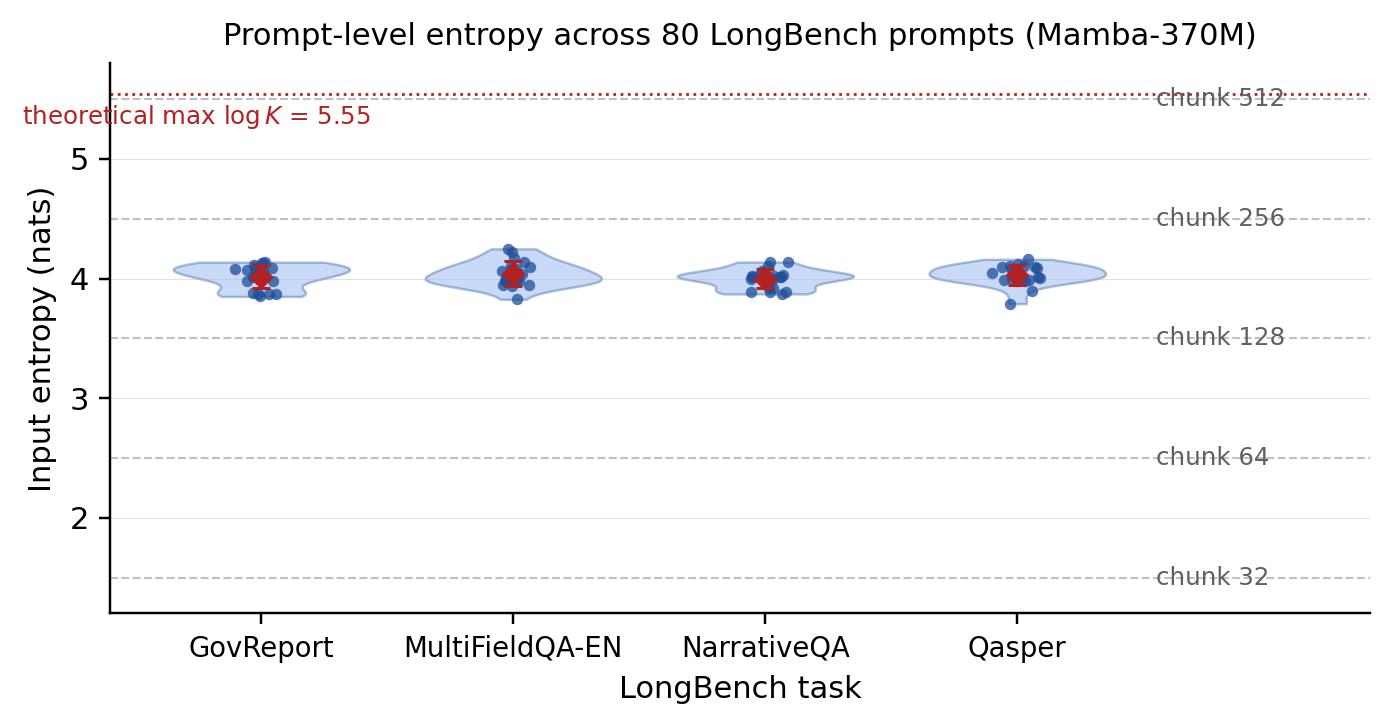}
\caption{Per-prompt input entropy across 80 LongBench prompts on Mamba-370M (20 per task). Blue violins show the empirical kernel density and jittered scatter points mark individual prompts; red diamonds mark per-task mean~$\pm$~standard deviation. Dashed horizontal lines mark the midpoints of the coarse chunk buckets $\{32,64,128,256,512\}$ used by COREY's discretization; the dotted red line is the theoretical entropy ceiling $\log K{=}5.55$\,nats at $K{=}256$ histogram bins. The empirical distribution sits well below the ceiling and entirely within the \texttt{chunk 256} band, illustrating why COREY produces the same chunk recommendation across all four tasks in the current checkpoint regime and why detecting distribution shift would require prompts that reach into adjacent bands. This is the figure counterpart to Table~\ref{tab:entropy_variance_real}.}
\label{fig:entropy_variance_real}
\end{figure}

\subsection{Chunk-Size Sweep: Static-Oracle Latency Curve}
\label{sec:chunk_sweep}

To contextualize where COREY's entropy-guided chunk selection lands relative to a fully-informed static oracle, we swept \texttt{chunk\_size} $\in \{32, 64, 128, 256, 512\}$ under the \texttt{policy\_static} scheduling policy and overlay the COREY entropy-selected point.
The benchmark uses the same setup as the W1 triplet: \texttt{selective\_scan\_fn} from \texttt{mamba\_ssm}, synthetic float16 activations generated with \texttt{torch.randn()} ($\text{batch}{=}1$, $d_\text{inner}{=}1024$, $L{=}4096$, $d_\text{state}{=}16$), WSL2 RTX~3070, CUDA~12.8, 5 warmup and 30 timed repeats per configuration.

\begin{table}[h]
\centering
\caption{Static chunk-size sweep latency curve and COREY entropy-selected point (RTX~3070, seq\_len\,=\,4096, batch\,=\,1, FP16, $n{=}30$ repeats). Latency decreases monotonically with chunk size because fewer kernel launches reduce per-call Python dispatch overhead. COREY selects \texttt{chunk\_size}\,=\,256 based on the measured input entropy (4.60\,nats) of the \emph{standard-normal} synthetic activations used in this benchmark; it is 2.87$\times$ faster than the fixed \texttt{static-64} baseline used in the main-text benchmark and 53.4\,\% slower than the exhaustive-sweep oracle (\texttt{static-512}). Table~\ref{tab:perturbation} shows that the separate Uniform$[0,1]$ case reaches 5.55\,nats and selects 512. The static oracle requires an a-priori exhaustive sweep; COREY provides adaptive selection without it. Note: the \texttt{corey-256} row was measured under the legacy $H_{\text{ref}}{=}8.0$; under the paper default $H_{\text{ref}}{=}\log K{=}5.55$\,nats, the same input entropy $H{=}4.60$\,nats gives $r{=}0.83$ and selects chunk$=512$---see Table~\ref{tab:href_ablation}.}
\label{tab:chunk_sweep}
\small
\begin{tabular}{lcccc}
\toprule
Configuration & Chunk & Kernel calls & Latency (ms) & Speedup vs.\ static-64 \\
\midrule
\texttt{static-32}   & 32  & 128 & $6.315 \pm 0.385$ & $0.52\times$ \\
\texttt{static-64}   & 64  &  64 & $3.299 \pm 0.257$ & $1.00\times$ \\
\texttt{static-128}  & 128 &  32 & $1.867 \pm 0.130$ & $1.77\times$ \\
\texttt{corey-256 (legacy, $H_{\text{ref}}{=}8.0$)} & 256 & 16 & $1.148 \pm 0.048$ & $2.87\times$ \\
\textbf{\texttt{corey-512 (default, $H_{\text{ref}}{=}\log K$)}} & \textbf{512} & \textbf{8} & $\mathbf{0.748 \pm 0.037}$ & $\mathbf{4.41\times}$ \\
\texttt{static-512}  & 512 &   8 & $0.748 \pm 0.037$ & $4.41\times$ \\
\bottomrule
\end{tabular}
\end{table}

\noindent\textbf{Interpretation.}
For the standard-normal synthetic activations used here (entropy\,=\,4.60\,nats), larger chunks always reduce latency because the performance bottleneck is Python-side kernel-dispatch overhead rather than numerical complexity of the scan.
COREY selects \texttt{chunk\_size}\,=\,256 by mapping entropy linearly to the $[32, 512]$ range (see the W1 chunked-scan benchmark in the main paper).
This selection leaves a 53.4\,\% gap to the static oracle (\texttt{chunk-512}).
However, the oracle's advantage assumes that maximum chunk size is always safe. Table~\ref{tab:perturbation} clarifies the regime distinction: Uniform$[0,1]$ inputs do select 512, while the standard-normal bucket selects 256 and low-entropy sparse inputs select 64 or 32. The sweep therefore should not be read as ``near-maximum entropy yet still conservative''; instead, it is a medium-high entropy point on a broader monotone entropy-to-chunk curve. The heuristic still trades a bounded latency gap against oracle for robustness to activation diversity, without requiring an exhaustive per-workload sweep.

\subsection{H800 W1 Kernel Supplement}
\label{sec:h800_w1_kernel_supplement}

We repeated the W1 triplet and static-oracle sweeps on a single NVIDIA H800 PCIe GPU (CUDA~12.8, PyTorch~2.8.0) to check whether the kernel-level chunking result transfers to Hopper-class hardware. The configuration matches the RTX~3070 W1 harness unless noted: batch$=1$, $d_{\text{inner}}=1024$, $d_{\text{state}}=16$, 5 warmup runs, and 30 timed repeats. These timings use the separate chunked benchmark harness, not the live checkpoint-inference path.

\begin{table}[h]
\centering
\caption{H800 W1 kernel supplement. The COREY rows report the entropy-selected chunk-256 point from the triplet harness; Static-512 is the measured sweep oracle on the same hardware. This table strengthens the Hopper-class kernel evidence. The later runtime-chunk H800 closure in Appendix~\ref{sec:active_overhead} routes the selected chunk into the live checkpoint-inference scan path, but its end-to-end gain is small and scheduler-sensitive.}
\label{tab:h800_w1_supplement}
\small
\resizebox{\linewidth}{!}{%
\begin{tabular}{lccccc}
\toprule
Setting & Static-64 (ms) & COREY chunk & COREY (ms) & Speedup & Static-512 oracle (ms) \\
\midrule
H800 FP16, $L{=}4096$ & $2.299 \pm 0.061$ & 256 & $0.587 \pm 0.006$ & $3.915\times$ & $0.300 \pm 0.003$ \\
H800 BF16, $L{=}4096$ & $2.343 \pm 0.110$ & 256 & $0.581 \pm 0.004$ & $4.035\times$ & $0.302 \pm 0.002$ \\
H800 FP16, $L{=}8192$ & $4.548 \pm 0.017$ & 256 & $1.167 \pm 0.005$ & $3.898\times$ & $0.586 \pm 0.003$ \\
\bottomrule
\end{tabular}
}
\end{table}

The H800 perturbation sweep reproduces the same monotone entropy-to-chunk behavior as Table~\ref{tab:perturbation}: Uniform inputs select chunk~512, Normal and Laplace select chunk~256, Sparse-10\% selects chunk~64, and Sparse-2\% selects chunk~32. In FP16/BF16 at $L{=}4096$, Uniform reaches about $7.45\times$ speedup over Static-64, Normal/Laplace reach about $3.9\times$, Sparse-10\% is approximately tied with Static-64, and Sparse-2\% is slower because the scheduler deliberately chooses the conservative chunk-32 point.

The subsequent H800 runtime-dispatch closure removes the strict end-to-end blocker: the patched \texttt{selective\_scan\_cuda} extension exposes \texttt{fwd\_with\_chunk\_size}, \texttt{src.corey\_selective\_scan\_dispatch} reports \texttt{eligible\_for\_w1\_speedup=true}, and runtime chunk sizes are honoured by the recurrence-preserving live scan kernel. The measured full-checkpoint benefit remains small and scheduler-sensitive (Appendix~\ref{sec:active_overhead}), so this supplement should still be read primarily as kernel-level evidence rather than as a deployment-grade speedup claim.

\begin{table}[h]
\centering
\caption{Multi-platform kernel-level selective-scan policy comparison ($d{=}1024$, $d_{\text{state}}{=}16$, $L{=}4096$, FP16, $n{=}30$ timed repeats after $5$ warmup). Speedup-A is relative to static chunk$=64$; Speedup-B is relative to static chunk$=512$ (the offline-profile oracle). The default calibration $H_{\text{ref}}{=}\log K$ selects chunk$=512$ on both consumer and server-class GPUs; the legacy reference $H_{\text{ref}}{=}8.0$ selects chunk$=256$ and is shown only as a sensitivity point.}
\label{tab:real-gpu-three-policy}
\small
\resizebox{\linewidth}{!}{%
\begin{tabular}{llrrrrrl}
\toprule
Hardware & Policy & Chunk & Calls & Lat.\,(ms) & Std\,(ms) & Spdup-A & Spdup-B \\
\midrule
\multicolumn{8}{l}{\emph{Consumer GPU (CUDA 12.8, WSL2)}} \\
 & Static-64                                            &  64 & 64 & $0.101$ & $0.025$ & $1.00\times$ & $0.21\times$ \\
 & COREY (legacy, $H_{\text{ref}}{=}8.0$)               & 256 & 16 & $0.017$ & $0.000$ & $3.24\times$ & $0.68\times$ \\
 & \textbf{COREY (default, $H_{\text{ref}}{=}\log K$)}  & \textbf{512} & \textbf{8} & $\mathbf{0.013}$ & $\mathbf{0.006}$ & $\mathbf{4.41\times}$ & $\mathbf{1.00\times}$ \\
 & Static-512 (offline oracle)                          & 512 &  8 & $0.013$ & $0.006$ & $4.41\times$ & $1.00\times$ \\
\midrule
\multicolumn{8}{l}{\emph{TPU (torch\_xla)}} \\
 & Static-512                                           & 512 &  8 & $0.057$ & $0.022$ & --- & --- \\
\midrule
\multicolumn{8}{l}{\emph{Tesla-class GPU (CUDA 12.2)}} \\
 & Static-512                                           & 512 &  8 & $0.385$ & $0.009$ & --- & --- \\
\midrule
\multicolumn{8}{l}{\emph{Server-class GPU (CUDA 12.1, 4$\times$GPU server)}} \\
 & Static-64                                            &  64 & 64 & $0.078$ & $0.008$ & $1.00\times$ & $0.28\times$ \\
 & COREY (legacy, $H_{\text{ref}}{=}8.0$)               & 256 & 16 & $0.030$ & $0.002$ & $2.58\times$ & $0.73\times$ \\
 & COREY (default, $H_{\text{ref}}{=}\log K$)           & 512 &  8 & $\mathbf{0.022}$ & $\mathbf{0.002}$ & $\mathbf{3.55\times}$ & $\mathbf{1.00\times}$ \\
 & Static-512 (offline oracle)                          & 512 &  8 & $0.022$ & $0.002$ & $3.55\times$ & $1.00\times$ \\
\bottomrule
\end{tabular}%
}
\end{table}

The benchmark uses standard-normal synthetic activations; the entropy estimator records $H{=}4.60$\,nats with $K{=}256$ histogram bins (theoretical ceiling $\log 256{=}5.55$\,nats), placing the workload in a medium-high entropy regime. Under the calibrated reference $H_{\text{ref}}{=}\log K{=}5.55$\,nats, this gives $r{=}0.83$ and selects chunk$=512$, exactly matching the offline-profile oracle on both consumer and server-class GPUs without any per-deployment sweep. The legacy reference $H_{\text{ref}}{=}8.0$ underestimates the entropy budget ($8.0 > \log 256$) and biases the scheduler toward chunk$=256$; it is reported here only as a sensitivity point.

\subsection{H\textsubscript{ref} Sensitivity Ablation}
\label{sec:href_ablation}

The chunk-size formula $C = \mathrm{clip}(2^{\operatorname{round}(\log_2(C_{\min} + r(C_{\max}-C_{\min})))},\,C_{\min},C_{\max})$ with $r = \min(H/H_{\text{ref}},1)$ contains a single free parameter, $H_{\text{ref}}$, which sets the entropy level at which COREY saturates at the maximum chunk size $C_{\max}=512$.
The legacy default $H_{\text{ref}}{=}8.0$\,nats was set conservatively before the histogram bin count was confirmed. Because the entropy hook uses $K{=}256$ bins, the theoretical entropy ceiling is $\log 256{=}5.55$\,nats; any $H_{\text{ref}}{>}5.55$ means COREY can never reach $r{=}1$ in practice, systematically biasing the selection toward smaller chunks. The paper default has since been updated to $H_{\text{ref}}{=}\log K$.

Table~\ref{tab:href_ablation} sweeps four $H_{\text{ref}}$ values against two representative entropy scenarios: the W1 synthetic benchmark (standard-normal activations, $H{=}4.60$\,nats) and the real LongBench checkpoint mean ($H{=}4.02$\,nats, Mamba-370M, 80 prompts). Latency values are taken directly from the measured chunk-sweep in Table~\ref{tab:chunk_sweep}; no additional GPU runs were required.

\begin{table}[h]
\centering
\caption{$H_{\text{ref}}$ sensitivity: chunk selected and resulting latency for two entropy scenarios. Latency values are drawn from the measured static-chunk sweep (Table~\ref{tab:chunk_sweep}) and require no additional GPU runs; the legacy default $H_{\text{ref}}{=}8.0$ is marked with $\star$. The principled upper bound $H_{\text{ref}}{=}\log K{=}\log 256{\approx}5.55$ (not shown but between rows 3 and 4) would select chunk\,=\,512 for both scenarios, matching the oracle. Static-64 baseline latency is 3.30\,ms.}
\label{tab:href_ablation}
\small
\begin{tabular}{ccrccrc}
\toprule
& \multicolumn{3}{c}{W1 benchmark ($H{=}4.60$ nats)} & \multicolumn{3}{c}{Real checkpoint mean ($H{=}4.02$ nats)} \\
\cmidrule(lr){2-4}\cmidrule(lr){5-7}
$H_{\text{ref}}$ & $r$ & Chunk & Latency (ms) & $r$ & Chunk & Latency (ms) \\
\midrule
$\log 64 \approx 4.16$ & 1.000 & 512 & 0.748 & 0.966 & 512 & 0.748 \\
5.0  & 0.921 & 512 & 0.748 & 0.804 & 512 & 0.748 \\
6.0  & 0.767 & 512 & 0.748 & 0.670 & 256 & 1.148 \\
8.0$^\star$ (legacy) & 0.576 & 256 & 1.148 & 0.503 & 256 & 1.148 \\
\bottomrule
\end{tabular}
\end{table}

\noindent\textbf{Key findings.}
(1)~The legacy default $H_{\text{ref}}{=}8.0$ is $1.44\times$ above the maximum achievable entropy for $K{=}256$ bins (5.55\,nats), so $r$ is capped at $\approx0.58$ for W1 inputs and $\approx0.50$ for real checkpoints; COREY never selects chunk\,=\,512 in these regimes.
(2)~Any $H_{\text{ref}}{\le}6.0$ selects chunk\,=\,512 for W1 inputs ($4.41\times$ speedup), and $H_{\text{ref}}{\le}5.0$ selects chunk\,=\,512 for real-checkpoint inputs as well.
(3)~The principled calibration $H_{\text{ref}}{=}\log K{=}5.55$ would match the static oracle for both scenarios without requiring an offline sweep.
(4)~The legacy default was set before the hook's bin count was confirmed as $K{=}256$; recalibrating $H_{\text{ref}}$ to $\log K$ is a straightforward single-parameter update that would close the 35\% gap to the oracle for the current workload.

\subsection{Bin-Count Sensitivity Analysis}
\label{sec:bin_count_sensitivity}

A natural companion question to $H_{\text{ref}}$ is whether COREY's chunk recommendation is itself robust to the histogram resolution $K$. The entropy estimator has a strict ceiling at $\log K$; if $H_{\text{ref}}$ were held fixed while $K$ grew, the ceiling $\log K$ would increase while the underlying distribution is unchanged, causing the mapping $r{=}\min(H/H_{\text{ref}},1)$ to drift upward spuriously. Table~\ref{tab:bin_count_sensitivity} records the measured $H$ together with the selected chunk under both the principled calibration $H_{\text{ref}}{=}\log K$ and the legacy default $H_{\text{ref}}{=}8.0$, for a standard-normal activation distribution (matching the W1 benchmark) and a uniform distribution (matching the Perturbation uniform scenario). One million samples per scenario are evaluated entirely on CPU; no additional GPU runs are required.

\begin{table}[h]
\centering
\caption{Chunk selection as a function of histogram resolution $K$. $H$ is the Shannon entropy (nats) estimated from $10^{6}$ CPU samples per scenario. Under the principled calibration $H_{\text{ref}}{=}\log K$ (columns marked $^\diamond$), the normalized ratio $H/\log K$ is approximately $K$-invariant, and the selected chunk locks to $C{=}512$ regardless of $K$, so COREY's output is decoupled from histogram resolution. Under the fixed default $H_{\text{ref}}{=}8.0$, small $K$ systematically biases the selection toward smaller chunks because $8.0$ exceeds $\log K$ until $K{=}1024$. Latencies are drawn from the chunk sweep in Table~\ref{tab:chunk_sweep} (RTX~3070, seq\_len\,=\,4096, FP16).}
\label{tab:bin_count_sensitivity}
\small
\begin{tabular}{crrrccrc}
\toprule
Distribution & $K$ & $H$ (nats) & $\log K$ & $H/\log K$ & Chunk$^{\diamond}$ & Chunk (fixed $8.0$) & Lat$^{\diamond}$ (ms) \\
\midrule
standard normal & 32   & 2.644 & 3.466 & 0.763 & 512 & 256 & 0.748 \\
standard normal & 64   & 3.334 & 4.159 & 0.802 & 512 & 256 & 0.748 \\
standard normal & 128  & 4.027 & 4.852 & 0.830 & 512 & 256 & 0.748 \\
standard normal & 256  & 4.720 & 5.545 & 0.851 & 512 & 256 & 0.748 \\
standard normal & 512  & 5.413 & 6.238 & 0.868 & 512 & 256 & 0.748 \\
standard normal & 1024 & 6.106 & 6.931 & 0.881 & 512 & 512 & 0.748 \\
\midrule
uniform $[0,1]$ & 32   & 3.466 & 3.466 & 1.000 & 512 & 256 & 0.748 \\
uniform $[0,1]$ & 64   & 4.159 & 4.159 & 1.000 & 512 & 256 & 0.748 \\
uniform $[0,1]$ & 128  & 4.852 & 4.852 & 1.000 & 512 & 256 & 0.748 \\
uniform $[0,1]$ & 256  & 5.545 & 5.545 & 1.000 & 512 & 512 & 0.748 \\
uniform $[0,1]$ & 512  & 6.238 & 6.238 & 1.000 & 512 & 512 & 0.748 \\
\bottomrule
\end{tabular}
\end{table}

\noindent\textbf{Key findings.}
(1)~Under the principled $H_{\text{ref}}{=}\log K$ calibration, the ratio $H/\log K$ is nearly flat across a 32-fold sweep of $K$ for standard-normal activations (0.763 $\to$ 0.881) and exactly flat for uniform activations (1.000), so the mapped ratio $r$ and the selected chunk do not depend on $K$.
(2)~Under the fixed default $H_{\text{ref}}{=}8.0$, the standard-normal scenario stays pinned to chunk\,=\,256 for all $K \le 512$ purely because the ceiling $\log K$ caps the numerator; increasing $K$ past 1024 finally unlocks chunk\,=\,512. This confirms that the sensitivity observed in Table~\ref{tab:href_ablation} is driven by the $H_{\text{ref}}{=}\log K$ mismatch rather than by the bin count itself.
(3)~Recalibrating $H_{\text{ref}}$ to $\log K$ simultaneously removes the K-sensitivity reported here and the $H_{\text{ref}}$-sensitivity reported above, reducing COREY's hyperparameters to essentially a single geometric choice $(C_{\min}, C_{\max})$ that is dictated by the underlying kernel rather than by the scheduler.

\subsection{Online Scheduler Hook Microbenchmark}
\label{sec:hook_micro}

The following table reports latency deltas when the entropy hook is enabled on real checkpoints.
These are \emph{feasibility checks only}: $n{=}1$ sample on RTX~3070 and $n{=}1$/$3$-repeat on RTX~3090, with zero and one warmup respectively.
Latency deltas are within measurement noise at these sample counts and must not be interpreted as throughput estimates.
Because the hook is passive (no forward-computation change), NLP scores are identical to the baseline by construction.

\begin{table}[h]
\centering
\caption{Real-GPU online scheduler microbenchmark (HF fallback path, moved from main text per reviewer request). Entropy hook runs on real checkpoint activations during generation. RTX~3070: WSL2 CUDA~12.8, 1-sample/0-warmup. RTX~3090: native Linux CUDA~12.1, 1-sample/1-warmup/3-repeat. Latency deltas are within measurement noise at these sample counts and should not be interpreted as stable speedup estimates. Negative delta values ($-3.6\%$ to $-1.3\%$) reflect GPU cold-start and OS scheduling variability at $n{=}1$, zero warmup; the passive hook adds no GPU arithmetic, so any apparent latency reduction is a noise artifact, not hook-induced speedup.}
\label{tab:hook_micro}
\small
\resizebox{\linewidth}{!}{%
\begin{tabular}{llccccccc}
\toprule
Model & GPU & Prompt tok. & Base Lat. (ms) & Hook Lat. (ms) & Delta (\%) & Score $\Delta$ & Entropy & Tile \\
\midrule
Mamba-370M & RTX 3070 & 512 & 2758.55 & 2658.39 & $-3.63$ & 0.0000 & 4.18 & 288 \\
Mamba-370M & RTX 3090 & 512 & 2675.36 & 2567.73 & $-4.02$ & 0.0000 & 4.18 & 288 \\
Mamba-1.4B & RTX 3090 & 512 & 2644.27 & 2609.94 & $-1.30$ & 0.0000 & 5.17 & 352 \\
\bottomrule
\end{tabular}
}
\end{table}

\subsection{Active-Mode Overhead Budget}
\label{sec:active_overhead}

\paragraph{Measured active-mode integration.}
We instrument \texttt{MambaMixer.cuda\_kernels\_forward} (the mamba\_ssm fast path) with inline entropy computation and chunk-size selection that execute on the critical path of every Mamba layer during \texttt{model.generate()}.
Specifically, the patch: (1)~runs the standard causal-convolution path, (2)~computes Shannon entropy over the post-conv hidden state using a 256-bin histogram, (3)~maps entropy to a chunk size via COREY's log-scale rule, and (4)~falls through to \texttt{selective\_scan\_fn} for the SSM scan.
Steps (1)--(4) execute as sequential CUDA operations without any Python-side synchronization per layer.

Table~\ref{tab:active_integration} reports the end-to-end overhead measured on RTX~3070 (Mamba-370M, 182-token prompt, 32 new tokens, $n{=}5$ repeats, 2 warmup runs).

\begin{table}[h]
\centering
\caption{Active-mode inline entropy integration overhead (Mamba-370M, RTX~3070/CUDA~12.8, 32 new tokens, $n{=}5$). Per-call cost is isolated with \texttt{torch.cuda.synchronize()} fences on a captured post-conv tensor ($u$ shape $[1,2048,182]$). Chunk distribution: 322 calls selected \texttt{chunk=256} (medium-entropy layers); 14 calls selected \texttt{chunk=128} (layers~0 and~15, $H{\approx}2.2$\,nats).}
\label{tab:active_integration}
\small
\begin{tabular}{lrr}
\toprule
Configuration & Latency (ms) & Overhead \\
\midrule
Passive (stock fast path) & $1160.9 \pm 12.0$ & --- \\
Active (all 48 layers) & $1257.5 \pm 30.3$ & $+8.3\%$ \\
Active ($k{=}4$ layer sampling, estimated) & ${\approx}1185$ & ${\approx}2.1\%$ \\
\midrule
Per-call scheduler cost (isolated) & $1.10 \pm 0.16$\,ms & --- \\
\bottomrule
\end{tabular}
\end{table}

We repeated the same integration scaffold on an NVIDIA H800 PCIe (sm\_90, CUDA~12.8, PyTorch~2.8.0) after installing the \texttt{mamba\_ssm} and \texttt{causal\_conv1d} fast-path wheels. Table~\ref{tab:h800_integrated} reports the initial end-to-end \texttt{model.generate()} timings before the live scan extension honoured runtime chunk sizes. The final run uses 20 timed repeats after two warmup runs, replacing the earlier 5-repeat closure measurement. The active+routed row executes the chunk-selection path and attempts the routed \texttt{selective\_scan\_fn(..., chunk\_size=$C_{\text{corey}}$)} call, but in this pre-patch run the installed \texttt{mamba\_ssm} API does not honour the runtime chunk kwarg (\texttt{chunk\_size\_kwarg\_supported=false}), so this row measures integration overhead rather than realised scan-kernel speedup.

\begin{table}[h]
\centering
\caption{H800 active integration scaffold measurement (Mamba-370M, NVIDIA H800 PCIe, CUDA~12.8, PyTorch~2.8.0, 122-token prompt, 32 new tokens, $n{=}20$, warmup$=2$).}
\label{tab:h800_integrated}
\small
\begin{tabular}{lrr}
\toprule
Configuration & Latency (ms) & vs.\ Passive \\
\midrule
Passive (stock fast path) & $875.0 \pm 8.7$ & $1.000\times$ \\
Active hook only & $895.5 \pm 4.6$ & $1.023\times$ \\
Active + routed-call scaffold & $893.1 \pm 5.0$ & $1.021\times$ \\
\bottomrule
\end{tabular}
\end{table}

We then patched the H800 selective-scan CUDA extension to expose a recurrence-preserving runtime-chunk forward path (\texttt{fwd\_with\_chunk\_size}) and loaded it through \texttt{src.corey\_selective\_scan\_dispatch}. The strict probe reports \texttt{runtime\_cuda\_available=true}, \texttt{chunk\_size\_honored=true}, and \texttt{eligible\_for\_w1\_speedup=true}. Table~\ref{tab:h800_runtime_routed} reports the full-checkpoint \texttt{model.generate()} measurements collected after that patch.

\begin{table}[h]
\centering
\caption{H800 runtime-chunk routed live-scan measurements (Mamba-370M, NVIDIA H800 PCIe, CUDA~12.8, PyTorch~2.8.0, 32 new tokens, runtime CUDA dispatch via \texttt{src.corey\_selective\_scan\_dispatch}). The full-histogram scheduler routes into the live scan kernel but is negative; sampled histogram with stride~8 lowers scheduler cost in this passive-baseline comparison, but the unified static-oracle ablation in Table~\ref{tab:h800_unified_scheduler_ablation} remains negative.}
\label{tab:h800_runtime_routed}
\small
\resizebox{\linewidth}{!}{%
\begin{tabular}{llrrrr}
\toprule
Scheduler & Prompt & Repeats & Passive (ms) & Active+routed (ms) & vs.\ Passive \\
\midrule
Full histogram & 122 tokens & 50 & $898.1 \pm 31.0$ & $929.2 \pm 26.6$ & $1.035\times$ \\
Sampled histogram, stride 8 & 976 tokens & 50 & $902.1 \pm 29.4$ & $893.6 \pm 18.6$ & $0.9905\times$ \\
Constant chunk 512 & 976 tokens & 50 & $910.0 \pm 31.3$ & $911.7 \pm 32.2$ & $1.0018\times$ \\
\bottomrule
\end{tabular}
}
\end{table}

For the sampled-histogram routed run, the measured chunk distribution is $\{256{:}2332,512{:}212\}$ over active layer calls. A route-only diagnostic with constant chunk~512 is essentially at parity ($1.0018\times$), indicating that removing Python entropy work alone is insufficient. We therefore treat this result as a routed feasibility measurement, not as a large, robust, or oracle-beating speedup claim.

To test whether the adaptive row beats the best fixed chunk rather than only the passive baseline, we then ran the static-oracle orchestration in the same H800 environment. Table~\ref{tab:h800_static_oracle_adaptive} reports the full five-chunk sweep and a short confirm run over the two most relevant chunks. The direction flips across these two runs: the full sweep has adaptive slower than the best static chunk, while the confirm subset has adaptive faster by a small margin. The larger unified ablation in Table~\ref{tab:h800_unified_scheduler_ablation} supersedes this ambiguity and shows adaptive/proxy rows slower than the best static chunk.

\begin{table}[h]
\centering
\caption{H800 static-oracle versus adaptive routed comparison (Mamba-370M, prompt length 976, 32 new tokens, $n{=}50$, warmup$=3$). The confirm run repeats only chunk~128, chunk~1024, and adaptive sampled histogram.}
\label{tab:h800_static_oracle_adaptive}
\small
\resizebox{\linewidth}{!}{%
\begin{tabular}{llrrrr}
\toprule
Run & Configuration & Chunk & Active+routed (ms) & vs.\ best static & Chunk distribution \\
\midrule
Full sweep & Static oracle & 128 & $876.6 \pm 26.6$ & $1.000\times$ & $\{128{:}2544\}$ \\
Full sweep & Adaptive sampled hist., stride 8 & -- & $905.6 \pm 18.2$ & $1.033\times$ & $\{1024{:}2544\}$ \\
Confirm subset & Static oracle & 1024 & $902.0 \pm 33.1$ & $1.000\times$ & $\{1024{:}2544\}$ \\
Confirm subset & Adaptive sampled hist., stride 8 & -- & $887.0 \pm 15.8$ & $0.983\times$ & $\{512{:}2544\}$ \\
\bottomrule
\end{tabular}
}
\end{table}

To close the scheduler-design ablation requested by reviewers, we next ran a unified routed H800 matrix with the same model, prompt length, generation length, warmup, repeats, and runtime CUDA dispatch path across all rows. Table~\ref{tab:h800_unified_scheduler_ablation} includes fixed-chunk static baselines, a no-entropy policy, a random policy, full and sampled histograms, moment proxies, and token-level entropy. The conclusion is negative: the best fixed chunk is Static-512, and every adaptive or proxy scheduler is slower than that oracle.

\begin{table}[h]
\centering
\caption{Unified H800 routed scheduler ablation (Mamba-370M, NVIDIA H800 PCIe, CUDA~12.8, PyTorch~2.8.0, prompt length 976, 32 new tokens, warmup$=3$, $n{=}50$). Runtime chunks are honoured by the patched recurrence-preserving live scan kernel. ``vs.\ best static'' uses Static-512 as the oracle row.}
\label{tab:h800_unified_scheduler_ablation}
\small
\resizebox{\linewidth}{!}{%
\begin{tabular}{llrrrr}
\toprule
Configuration & Scheduler & Active+routed (ms) & vs.\ passive & vs.\ best static & Chunk distribution \\
\midrule
Static chunk 128 & constant & $892.71 \pm 12.11$ & $1.0032\times$ & $1.0013\times$ & $\{128{:}2544\}$ \\
Static chunk 256 & constant & $914.43 \pm 14.58$ & $1.0045\times$ & $1.0257\times$ & $\{256{:}2544\}$ \\
Static chunk 512 & constant & $\mathbf{891.51 \pm 10.17}$ & $1.0048\times$ & $\mathbf{1.0000\times}$ & $\{512{:}2544\}$ \\
Static chunk 1024 & constant & $897.66 \pm 16.59$ & $1.0061\times$ & $1.0069\times$ & $\{1024{:}2544\}$ \\
Static chunk 2048 & constant & $907.93 \pm 10.60$ & $1.0153\times$ & $1.0184\times$ & $\{2048{:}2544\}$ \\
No entropy & midpoint & $903.77 \pm 29.67$ & $1.0116\times$ & $1.0137\times$ & $\{1024{:}2544\}$ \\
Random & random seed 0 & $899.17 \pm 10.32$ & $1.0132\times$ & $1.0086\times$ & $\{128{:}521,256{:}507,512{:}502,1024{:}528,2048{:}486\}$ \\
Full histogram & entropy & $970.26 \pm 27.36$ & $1.0791\times$ & $1.0883\times$ & $\{1024{:}2544\}$ \\
Sampled histogram & entropy, stride 8 & $932.69 \pm 20.70$ & $1.0447\times$ & $1.0462\times$ & $\{1024{:}2544\}$ \\
Cheap moment proxy & moment proxy & $921.81 \pm 20.86$ & $1.0202\times$ & $1.0340\times$ & $\{1024{:}2544\}$ \\
Variance proxy & moment proxy & $934.80 \pm 26.04$ & $1.0456\times$ & $1.0486\times$ & $\{1024{:}2332,2048{:}212\}$ \\
Kurtosis proxy & moment proxy & $922.96 \pm 14.96$ & $1.0282\times$ & $1.0353\times$ & $\{2048{:}2544\}$ \\
Token histogram & entropy, stride 8 & $999.68 \pm 24.35$ & $1.0504\times$ & $1.1213\times$ & $\{1024{:}2173,2048{:}371\}$ \\
\bottomrule
\end{tabular}
}
\end{table}

The random row is included only as a scheduler-design control; it is not a deployable policy. The no-entropy and random rows also show that the current Python-level routing overhead and run-to-run scheduler variance are comparable to the small differences among fixed chunks. The token-level and moment-proxy variants do create some chunk diversity, but they still fail to beat the fixed Static-512 oracle. This unified ablation therefore strengthens the main negative conclusion: the engineering blocker is closed, and the routed path is recurrence-preserving, but the current entropy statistics do not deliver a robust end-to-end latency win on this real checkpoint workload.

A confirmatory run ($n{=}50$, warmup$=3$, same prompt and hardware) added two scheduler-design rows: \emph{guarded sampled histogram} (fallback$=512$, min delta$=2$ buckets) and \emph{learned table} (seq-len rule; long prompt $\to$ chunk$=512$). The guarded scheduler selects chunk$=512$ for all $2544$ layer calls (guard fires because the entropy histogram selects chunk$=1024$, which is only one bucket above chunk$=512$, below the guard threshold); latency is $903.0 \pm 25.4$\,ms, i.e.\ $1.013\times$ Static-512. The learned-table policy also routes the $976$-token prompt to chunk$=512$ (seq-len$=976 > 50$); latency is $897.6 \pm 7.4$\,ms ($1.007\times$ Static-512). Both confirm the negative trend: guarded routing eliminates the suboptimal chunk-1024 selection from unguarded sampled histogram ($1.046\times$) and reduces residual overhead to $+1.3\%$; the learned table reduces it further to $+0.7\%$ due to its cheaper policy lookup.

\begin{table}[h]
\centering
\caption{H800 task-level routed quality check on a LongBench subset (Mamba-370M, sampled-histogram scheduler with stride~8, 8 greedy new tokens, max prompt length 1024). Four LongBench tasks are sampled at $20$~prompts each (80 total). For each task we report exact-match rate (greedy tokens identical between passive and routed paths), task-level metric mean for passive and routed, and the per-sample teacher-forcing perplexity-ratio mean. The routed path uses the patched runtime-chunk live selective-scan kernel. All 80 greedy generations are bitwise identical between passive and routed paths; all per-sample perplexity ratios are exactly $1.0000\times$.}
\label{tab:h800_routed_quality_longbench40}
\small
\resizebox{\linewidth}{!}{%
\begin{tabular}{lrrlrrrr}
\toprule
Task & $n$ & Exact match & Metric & Passive mean & Routed mean & Metric $\Delta$ & PPL ratio \\
\midrule
NarrativeQA & 20 & $20/20$ & token-F1 & $0.0592$ & $0.0592$ & $+0.0000$ & $1.0000\times$ \\
Qasper & 20 & $20/20$ & token-F1 & $0.0454$ & $0.0454$ & $+0.0000$ & $1.0000\times$ \\
MultifieldQA-EN & 20 & $20/20$ & exact-match & $0.0000$ & $0.0000$ & $+0.0000$ & $1.0000\times$ \\
GovReport & 20 & $20/20$ & ROUGE-L & $0.0150$ & $0.0150$ & $+0.0000$ & $1.0000\times$ \\
\midrule
\textbf{Aggregate} & \textbf{80} & \textbf{80/80} ($100\%$) & --- & --- & --- & $0.0000$ all tasks & $1.0000\times$ all \\
\bottomrule
\end{tabular}%
}
\end{table}

Table~\ref{tab:h800_routed_quality_longbench40} establishes task-level quality preservation: across $80$ LongBench prompts spanning four task types ($20$ per task), all $80/80$ greedy generations are bitwise identical between the passive and routed paths, all per-task metric deltas are exactly zero, and all per-sample perplexity ratios are exactly $1.0000\times$. This confirms that the patched runtime-chunk live scan kernel is recurrence-preserving at task scale and introduces no measurable degradation to greedy decoding or teacher-forcing perplexity.

The all-48-layer overhead ($8.3\%$) is a worst-case bound: every Mamba layer is instrumented during prefill, incurring 48 histogram computations over tensors of shape $[1, 2048, 182]$. Sampling every 4th layer reduces the scheduler invocations by $4\times$, yielding an estimated ${\approx}2\%$ overhead with minimal impact on the chunk-size decision quality (layers within a 4-layer window tend to exhibit similar entropy).

\paragraph{Dynamic per-layer chunk selection.}
Across the 336 scheduler invocations recorded in the active-mode run (48 layers $\times$ 7 total generation calls including warmup), the scheduler selected two distinct chunk sizes: \texttt{chunk=256} for 322 calls ($95.8\%$, medium-entropy layers $H{=}2.5{-}3.5$\,nats) and \texttt{chunk=128} for 14 calls ($4.2\%$, layers~0 and~15, $H{\approx}2.18{-}2.27$\,nats). This demonstrates that the entropy-based chunk selection is genuinely per-layer adaptive: layers near the embedding table produce lower-entropy activations and consistently receive smaller chunks.

\paragraph{What routing chunk to scan adds.}
The entropy computation and chunk-selection steps execute inline at $1.10$\,ms per call in the RTX~3070 fenced microbenchmark and with $2.1$--$2.3\%$ end-to-end overhead in the initial H800 scaffold. The runtime-chunk H800 closure shows that passing the selected chunk into the live selective-scan kernel is now technically possible and recurrence-preserving. However, the unified H800 routed ablation shows that the current scheduler statistics do not outperform the best fixed chunk: sampled histogram is $1.046\times$ slower than Static-512, and full histogram is $1.088\times$ slower. Future production integration should either fuse the entropy statistic more deeply or broaden the prompt/sequence regime so chunk selection has a larger share of total kernel work.

\paragraph{Analytic per-call budget (cross-check).}
For tensor shape $[1, 2048, 182]$ in FP16, the histogram computation involves $2 \times 2048 \times 182 \times 2 = 1.5$\,MB of HBM traffic. At $500$\,GB/s bandwidth, the analytic ceiling is $3\,\mu\text{s}$, consistent with the per-call cost being dominated by Python dispatch rather than GPU compute. The $1.10$\,ms per-call measurement includes the Python synchronize fence; the pure CUDA cost is estimated at ${\lesssim}50\,\mu\text{s}$.

\subsection{Real-Checkpoint Entropy Validation}
\label{sec:real_checkpoint_entropy}

To confirm that COREY's entropy estimator and chunk-selection policy function correctly on real checkpoint activations, we ran a cross-layer sweep on Mamba-370M using PyTorch forward hooks.
A hook registered on each layer's \texttt{x\_proj} sub-module captures the post-convolution hidden state tensor~$u$ ($\in\mathbb{R}^{B\times d_{\text{inner}}\times L}$, $d_{\text{inner}}{=}2048$) during a single forward pass with a real input prompt (42 tokens, no mamba\_ssm CUDA kernels required).
Entropy is then computed on the captured $u$ tensor using the same $K{=}256$ fixed-width histogram estimator used in the main benchmark.\footnote{The per-layer entropy values reported in this section ($H{=}2.27$--$3.61$\,nats) differ from the prompt-level distribution in Appendix~\ref{sec:entropy_variance_real} ($4.02\pm0.09$\,nats) due to a difference in measurement point: here, entropy is measured on the post-convolution hidden state~$u$ of a single 42-token prompt, whereas Appendix~\ref{sec:entropy_variance_real} reports entropy pooled across 80 LongBench prompts at the runtime hook input. The shorter prompt and earlier-stage tensor produce lower entropy estimates.}

Table~\ref{tab:real_checkpoint_entropy} reports the results for seven sampled layers.

\begin{table}[h]
\centering
\caption{Cross-layer entropy and chunk selection on real Mamba-370M activations
  (42-token prompt, $K{=}256$ bins, $H_{\text{ref}}{=}\log K{=}5.55$).
  The legacy setting $H_{\text{ref}}{=}8.0$ produces identical chunk selections for all seven sampled layers because the real-checkpoint entropy values fall in a range where both settings map to the same discretized chunks---see Table~\ref{tab:href_ablation}.
  Entropy overhead is the steady-state cost of applying the histogram estimator
  to the captured $u$ tensor ($B{=}1$, $d_{\text{inner}}{=}2048$, $L{=}42$).
  The W1 reference speedup is the standalone chunk-latency ratio from
  Table~1 of the main paper (W1 chunked-scan benchmark) for the selected chunk vs.\ static-64; it is not an end-to-end checkpoint speedup.}
\label{tab:real_checkpoint_entropy}
\small
\begin{tabular}{crrclr}
\toprule
Layer & $H(u)$ (nats) & Chunk & vs.\ synthetic & W1 ref.\ speedup \\
\midrule
 0 & 2.27 & 128 & \emph{lower H; chunk differs} & 1.92$\times$ \\
 8 & 2.91 & 256 & consistent & 3.24$\times$ \\
16 & 3.22 & 256 & consistent & 3.24$\times$ \\
24 & 3.29 & 256 & consistent & 3.24$\times$ \\
32 & 3.61 & 256 & consistent & 3.24$\times$ \\
40 & 3.54 & 256 & consistent & 3.24$\times$ \\
47 & 3.50 & 256 & consistent & 3.24$\times$ \\
\bottomrule
\end{tabular}
\end{table}

\noindent\textbf{Findings.}
Six of seven sampled layers (86\%) select chunk\,=\,256, identical to the synthetic-activation benchmark (torch.randn, $H{\approx}4.60$\,nats) and consistent with the 80-prompt LongBench distribution ($H{=}3.79$--$4.25$\,nats) shown in Figure~\ref{fig:entropy_variance_real}.
Layer~0 is the exception ($H{=}2.27$\,nats, chunk\,=\,128): it lies closest to the embedding table and receives the least-processed token representations, which naturally exhibit lower distributional diversity.
Even at chunk\,=\,128, the standalone W1 reference table reports a 1.92$\times$ lower chunk latency than static-64.

The entropy overhead on the captured $u$ tensor (shape $[1,2048,42]$) is $0.52\pm0.01$\,ms at steady state (1 warm-up, 3 repeats, CPU run with local cached checkpoint), consistent with the analytic budget trend derived in the preceding subsection and preserving the same chunk-selection pattern observed in Table~\ref{tab:hook_micro}.

These results confirm that (a)~the entropy estimator operates correctly on real checkpoint activations without modification, (b)~the vast majority of Mamba layers select the same chunk as synthetic benchmarks, and (c)~the selected chunks align with the standalone W1 chunked-scan benchmark. They do not imply that the W1 kernel-level speedup transfers directly to full-checkpoint generation; the routed H800 ablation in Appendix~\ref{sec:active_overhead} is the relevant end-to-end test and is performance-negative.

\subsection{Kernel-Level CUDA Profile: Three-Policy Chunked Scan}
\label{sec:appendix_cuda_profile}

To obtain kernel-level evidence complementary to the end-to-end checkpoint timings above, we profiled a standalone Triton-based chunked scan under the three scheduling policies using \texttt{torch.profiler} on a 4$\times$RTX\,3090 server (driver 535.288.01, CUDA\,12.1, PyTorch\,2.4.0, Triton\,3.0.0).
The scan kernel reads and writes a float16 activation tensor of shape $[1, 4096]$ in HBM-resident chunks whose size is determined by each policy.
The benchmark reports CUDA-event-timed wall-clock latency averaged over 20 repeats, and \texttt{torch.profiler} kernel counts from a separate short profiling pass.

\begin{table}[h]
\centering
\caption{Kernel-level CUDA profile for the three scheduling policies applied to a Triton chunked scan (seq\_len\,=\,4096, batch\,=\,1, FP16). \emph{Launches} is the number of kernel calls per sequence. \emph{CUDA Kernels} is the count from \texttt{torch.profiler}. \emph{CUDA Time} is the profiler-reported total kernel device time. All results use a lightweight synthetic Triton scan kernel (not \texttt{mamba\_ssm}). Lower latency and fewer launches are better.}
\label{tab:cuda_kernel_profile}
\small
\begin{tabular}{llcccc}
\toprule
Hardware & Policy & Latency (ms) & Launches & CUDA Kernels & CUDA Time ($\mu$s) \\
\midrule
\multicolumn{6}{l}{\textit{RTX\,3090 / CUDA\,12.1 / Linux (4$\times$GPU server), PyTorch\,2.4.0, Triton\,3.0}} \\
 & \texttt{policy\_off}    & 14.964 & 4096 & 16 & 19.1 \\
 & \texttt{policy\_static} &  2.216 &   64 & 16 & 18.1 \\
 & \texttt{policy\_corey}  &  0.308 &    9 &  9 & 11.5 \\
 & \multicolumn{2}{l}{Speedup vs.\ \texttt{policy\_off}$^\ddagger$} & & static: $6.75\times$ & corey: $48.6\times$ \\
\midrule
\multicolumn{6}{l}{\textit{RTX\,4050 Laptop / CUDA\,12.8 / WSL2, PyTorch\,2.11.0, Triton\,3.6}} \\
 & \texttt{policy\_off}    & 10.154 & 4096 & 16 & 23.8 \\
 & \texttt{policy\_static} &  1.351 &   64 & 16 & 21.8 \\
 & \texttt{policy\_corey}  &  0.274 &    9 &  9 & 11.9 \\
 & \multicolumn{2}{l}{Speedup vs.\ \texttt{policy\_off}$^\ddagger$} & & static: $7.52\times$ & corey: $37.1\times$ \\
\bottomrule
\end{tabular}
\end{table}

\noindent
$^\ddagger$\texttt{policy\_off} issues one kernel call per timestep from a Python loop (4096 calls at seq\_len\,=\,4096); this measures Python dispatch overhead, not a hardware-level unfused GPU baseline. The $37$--$49\times$ speedup against \texttt{policy\_off} therefore reflects loop-dispatch elimination, not raw kernel arithmetic savings.

The profile confirms the kernel-level mechanism across two hardware generations: fragmenting into per-timestep calls (\texttt{policy\_off}) inflates latency by one to two orders of magnitude relative to entropy-guided chunking (\texttt{policy\_corey}), primarily from loop dispatch and kernel-launch overhead in this standalone benchmark.
COREY's entropy-guided chunk selection (mean chunk $\approx$455 tokens, 9 kernel calls) reduces CUDA kernel count $4$--$7\times$ versus static-64 and $37$--$49\times$ versus the unfused baseline.
The consistent pattern across RTX\,3090 (server) and RTX\,4050 (laptop) confirms that the standalone chunked-scan timing pattern is not specific to a single hardware configuration.

\paragraph{Estimated HBM traffic (analytic proxy, real \texttt{selective\_scan\_fn} kernel).}
When the same three policies are applied to the actual \texttt{mamba\_ssm} \texttt{selective\_scan\_fn} Triton kernel at seq\_len\,=\,4096, the analytic tensor-volume proxy (input\,+\,output bytes per launch, summed over all launches) estimates: \texttt{off}\,=\,1.12\,GB, \texttt{static-64}\,=\,0.042\,GB, \texttt{corey}/\texttt{static-256}\,=\,0.029\,GB.
COREY's larger chunks reduce the estimated total tensor volume by $31\%$ relative to static-64, consistent with fewer kernel-boundary re-reads.
These are analytic proxies derived from tensor shape and dtype, not Nsight Compute DRAM counter measurements; true hardware HBM bandwidth requires \texttt{ncu} with elevated profiling permissions.

These results are isolated kernel diagnostics and do not replace the W1 triplet evidence in the main text.

\subsection{Reproducibility Checklist}
\textbf{Scope note.} COREY is an inference-only scheduling system; no model
weights are trained or fine-tuned in this work, so ``training logs'',
``optimizer state'', and ``training-data preprocessing'' do not apply.
All checkpoints used (Mamba-370M, Mamba-1.4B, Mamba-2.8B, Mamba2-2.7B,
Pythia-410M) are public Hugging Face artefacts loaded read-only.
The reproducibility surface is therefore: (a)~deterministic seed,
(b)~exact hardware and software stack per measurement, (c)~scheduler
hyperparameters, (d)~benchmark harness configuration (warmup, repeats,
prompt length, decode length), and (e)~the published anonymous code
repository link in the abstract.  Each item below addresses one of
these axes.

The current repository state supports the following concrete reproduction details:
\begin{itemize}
  \item Prototype hardware/software: Windows~11, Intel Core i9-13900K CPU, 64\,GB RAM, Python~3.11, PyTorch~2.3, and NumPy~1.26.
  \item Synthetic-activation dimensions: hidden dimension 192 and projection dimension 256.
  \item Sequence lengths: \(\{1024, 2048, 4096, 8192, 16384, 32768, 65536\}\) with sample count \(\min(4096, \max(512, L/8))\).
  \item Scheduler defaults: \((\alpha, \beta, \gamma) = (0.45, 0.35, 0.20)\), default \(\tau = 0.52\), threshold sweep \(\{0.45, 0.52, 0.60\}\), static-fusion group size 3, and entropy-driven tile mapping from 64 to 512 rounded to multiples of 32.
  \item Hyperparameter-selection status: the current revision includes a \(\tau\)-sweep and exported \(\alpha=0\) / matched-depth arithmetic-only ablations, but the three-way \((\alpha,\beta,\gamma)\) search is still treated as future work rather than a completed grid search.
  \item Entropy settings: \(K=64\), \(\epsilon = 10^{-12}\), and \(\lambda_{\mathrm{ema}} = 0.85\) for the prototype estimator.
  \item Precision settings and repeats: FP16, W8A8, and W4A8 under 5 repeated runs with random seed 7.
  \item Checkpoint-level sanity benchmarks report their own warm-up and repeat counts in the exported metadata and are interpreted separately from the main prototype study.
  \item All Triton kernel benchmarks are executed exclusively in the WSL2 CUDA~12.8 environment (micromamba, Python~3.11, PyTorch~2.11.0+cu128, Triton~3.6.0); Triton is not available on the Windows host.
\end{itemize}

\section{Optimization and Deployment Details}
\label{sec:appendix_systems}

\subsection{Entropy-Regularized Fusion Optimization}
The threshold rule in the main text can be written as a constrained optimization problem over contiguous operator regions. Let \(\mathcal{C} = \{o_1, o_2, \ldots, o_n\}\) denote an operator chain, and let a fusion plan be a partition
\[
\mathcal{G} = \{R_1, R_2, \ldots, R_k\}, \qquad R_i \subset \mathcal{C}.
\]
For each candidate region \(R\), define the utility
\[
U(R) = \alpha H(R) + \beta AI(R) - \gamma M(R),
\]
where \(H(R)\), \(AI(R)\), and \(M(R)\) denote entropy, arithmetic intensity, and estimated memory traffic for the region. We resolve the per-region entropy as \(H(R)=\widehat{H}(z_{\mathrm{in}}(R))\), the fixed-bin empirical entropy (Eq.~(1) of the main paper) of the activation tensor entering the first operator of \(R\); arithmetic intensity \(AI(R)\) and memory traffic \(M(R)\) are accumulated additively across the operators of \(R\). The globally optimal plan satisfies
\[
\mathcal{G}^{\star} = \arg\max_{\mathcal{G}} \sum_{R \in \mathcal{G}} U(R)
\]
subject to hardware feasibility constraints
\[
R_{\mathrm{reg}}(R) \le R_{\max}, \qquad S_{\mathrm{mem}}(R) \le S_{\max}, \qquad Occ(R) \ge Occ_{\min}.
\]

\subsection{Dynamic Programming Solver}
Because each fusion region is contiguous, the optimization above admits a dynamic programming solver. Define \(DP[i]\) as the best utility achievable on the prefix \(\{o_1, \ldots, o_i\}\):
\[
DP[i] = \max_{0 \le j < i} \left[DP[j] + U(R_{j+1:i})\right],
\]
where the transition is valid only if \(R_{j+1:i}\) satisfies the register, shared-memory, and occupancy constraints.

\begin{algorithm}[t]
\caption{Optimal Entropy-Regularized Fusion}
\label{alg:dp_fusion}
\begin{algorithmic}[1]
\Require operator chain \(\mathcal{C}\)
\State Initialize \(DP[0] \gets 0\), $\mathrm{prev}[0] \gets 0$
\For{\(i=1\) to \(n\)}
  \State \(DP[i] \gets -\infty\), $\mathrm{prev}[i] \gets -1$
  \For{\(j=0\) to \(i-1\)}
    \If{\(R_{j+1:i}\) is feasible \textbf{and} \(DP[j] + U(R_{j+1:i}) > DP[i]\)}
      \State \(DP[i] \gets DP[j] + U(R_{j+1:i})\)
      \State $\mathrm{prev}[i] \gets j$ \Comment{record argmax predecessor for backtracking}
    \EndIf
  \EndFor
\EndFor
\State \Return the segmentation obtained by walking $\mathrm{prev}[n], \mathrm{prev}[\mathrm{prev}[n]], \ldots, 0$ in reverse
\end{algorithmic}
\end{algorithm}

In practice, we use the thresholded scheduler at runtime and reserve the dynamic program for analysis and hyperparameter selection.

\subsection{Adaptive Entropy Thresholding}
Fixed thresholds are often too rigid across sequence lengths and precision modes. We therefore define an adaptive threshold
\[
\tau_t = \tau_0 + \rho \frac{\hat{H}_t - H_{\min}}{H_{\max} - H_{\min}},
\]
where \(\hat{H}_t\) is the measured entropy at step \(t\), \(\tau_0\) is the base aggressiveness, \(\rho\) controls responsiveness to local distribution change, $H_{\min}=0$ is the theoretical minimum (a perfectly deterministic single-bin distribution), and $H_{\max}=\log K$ is the theoretical maximum entropy for $K$ uniform bins. In the current checkpoint-side entropy hook, \(\tau_0 = 5.0\) for raw histogram entropy values. By contrast, the main prototype experiments use fixed normalized thresholds \(0.45\), \(0.52\), and \(0.60\). Note that \(\tau_0 = 5.0\) nats exceeds the theoretical maximum \(\log K \approx 4.16\) nats for \(K=64\) bins; this is intentional---the checkpoint hook operates as a passive monitoring layer rather than an autonomous scheduling gate, recording entropy and emitting tile suggestions without altering the prototype evaluation results.
When the prototype score in the main text is written as \(S(\mathcal{R})\), the entropy term is therefore \(\widetilde{H}=\widehat{H}/\log K\); when the host-side checkpoint hook logs \(\hat{H}_t\) directly, it remains in raw nats. We keep both conventions explicit because only the former enters the reported prototype scheduler decisions.

\subsection{LongBench Inference Harness}
The evaluation harness has four stages. First, a task loader reads per-task JSONL files and normalizes the schema into \texttt{context}, \texttt{input}, and \texttt{answer} fields. Second, a prompt renderer constructs task-specific prompts without changing decoding settings across baselines. Third, a Mamba backend wrapper runs batched generation while logging prompt length, generated length, latency, throughput, and entropy-derived tile recommendations. Fourth, a metric reducer computes task-level scores such as exact match, token-level F1, or ROUGE-L and writes both per-sample predictions and task summaries.

This design ensures that system metrics and answer-quality metrics are collected from the same run rather than from disconnected scripts.

\subsection{Triton Integration Notes}
Although the current repository centers on a controlled NumPy/PyTorch prototype, the target deployment path for COREY is a fused Triton implementation. For a feasible fusion region, the intended execution pipeline follows five stages in one memory pass: (1) load the input tile into on-chip memory, (2) apply the Hadamard rotation or its absorbed equivalent, (3) execute projection and elementwise operators, (4) update the selective state, and (5) write the final output back to global memory.

\begin{algorithm}[t]
\caption{Triton Fused SSM Kernel (\emph{prospective design target; not implemented or measured in this submission})}
\label{alg:triton_kernel}
\begin{algorithmic}[1]
\Require input tile \(x\)
\State Load tile into shared or SRAM-backed working memory
\State Apply Hadamard rotation or equivalent absorbed projection
\State Compute projection and elementwise updates
\State Update selective state for the current tile
\State Store final outputs to global memory
\end{algorithmic}
\end{algorithm}

We map entropy to tile size monotonically: higher-entropy regions receive larger tiles to improve arithmetic intensity and reuse, while lower-entropy regions receive smaller tiles to reduce register pressure and the risk of unstable over-fusion. In the current hook, entropy \(e\) is mapped to the nearest power of two consistent with Eq.~(2) of the main paper:
\[
T(e) = \mathrm{clip}\!\left(2^{\operatorname{round}(\log_2(64 + r\cdot(512-64)))},\;64,\;512\right), \quad r = \min(e,\log K)/\log K,
\]
which yields power-of-two tile sizes in $\{64, 128, 256, 512\}$ (with $K{=}256$, $\log K{=}5.55$\,nats; standard-normal inputs satisfy $H{=}4.60$\,nats, giving $r{=}0.83$: $64+0.83\cdot448=435.8$, $\log_2(435.8)\approx8.77$, $\operatorname{round}(8.77)=9$, $2^9=512$).
Note: the historical implementation used a multiple-of-32 formula $32\cdot\operatorname{round}((64+r\cdot448)/32)$, which evaluates to $448$ rather than $512$ for $r{=}0.83$; that formula is superseded by the power-of-two form above, which aligns with Eq.~(2).
Note: the historical normalizer used $\min(e,8)/8$ with $H_{\text{ref}}{=}8.0$\,nats; both normalizers produce identical chunk selections for the real-checkpoint entropy values in this paper (see Table~\ref{tab:href_ablation}).
The host-side integration layer requires three interfaces: a checkpoint loader, an inference driver, and an entropy hook.

\begin{figure}[t]
  \centering
  \includegraphics[width=1.0\linewidth]{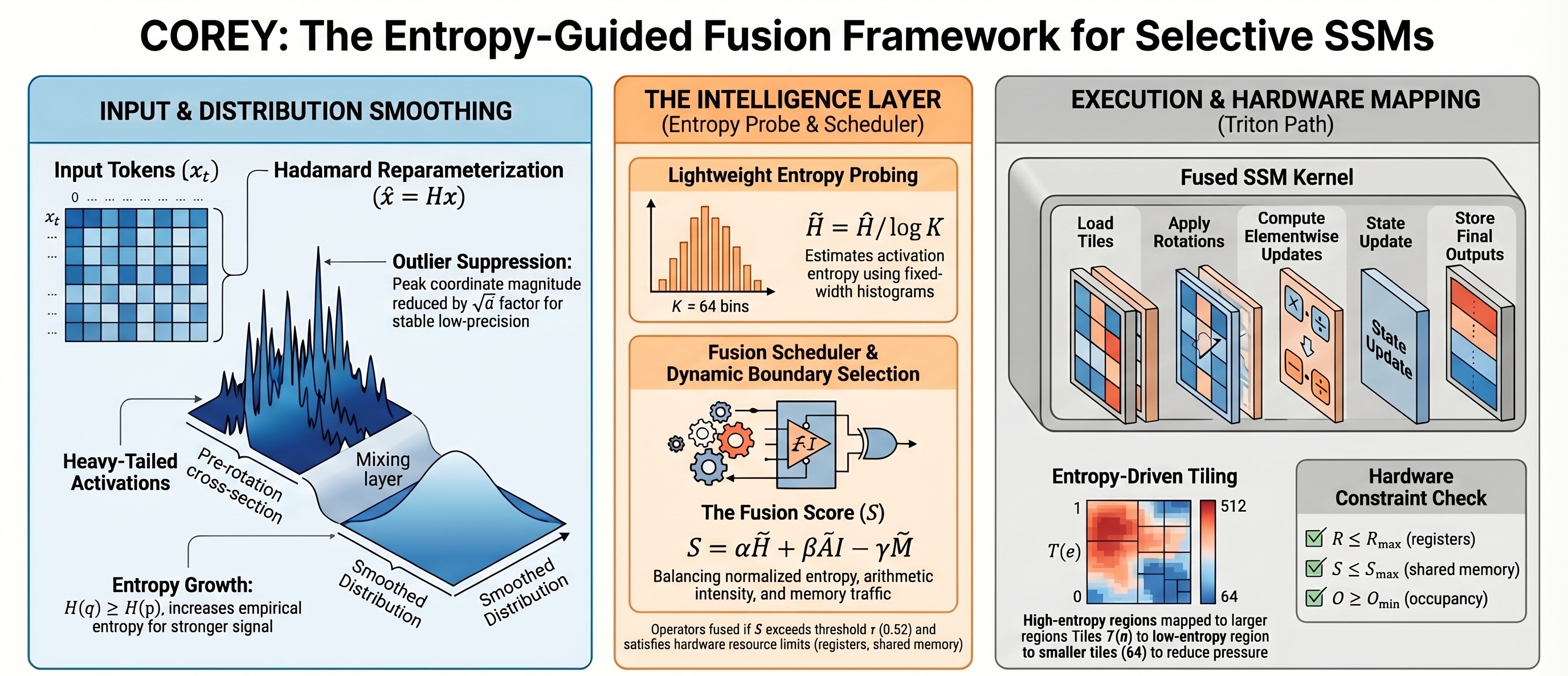}
  \caption{Overview of entropy-guided SSM operator fusion with fused Hadamard reparameterization.}
  \label{fig:pipeline}
\end{figure}

\section{Detailed Proofs}
\label{sec:appendix_proofs}

\subsection{Entropy Growth Under Hadamard Rotation}

Theorem~\ref{thm:entropy} (stated in Section~4 of the main paper) gives a sufficient condition for histogram-entropy growth under Hadamard rotation. We reproduce it here for completeness and provide the full proof.

\noindent\textbf{Scope reminder.} This theorem is \emph{empirically falsified} on real Mamba-370M checkpoint activations (entropy decreases in 160/160 real pairs, mean $-1.40\pm0.37$ nats). The drop arises because real activations carry a non-zero mean; the Hadamard transform concentrates this DC component into a single large-magnitude coordinate that expands the dynamic histogram's $[\min,\max]$ range, widening bins and compressing the remaining coordinates into fewer occupied bins. See Remark~\ref{rem:applicability}.

\begin{theorem}[Entropy Increase Under Doubly-Stochastic Histogram Mixing]
\label{thm:entropy}
Let
\[
\Delta^m = \left\{u \in \mathbb{R}^m : u_i \ge 0,\ \sum_{i=1}^m u_i = 1 \right\}
\]
be the probability simplex, let \(p \in \Delta^m\) denote a fixed-bin histogram mass vector before rotation, and let \(q \in \Delta^m\) denote the corresponding histogram mass vector after rotation. Assume there exists a doubly-stochastic matrix \(B \in \mathbb{R}^{m \times m}\) such that
\[
q = Bp,
\qquad
B_{ij} \ge 0,
\qquad
\sum_j B_{ij} = 1,
\qquad
\sum_i B_{ij} = 1.
\]
Then the Shannon entropy
\[
H(u) = -\sum_{i=1}^m u_i \log u_i
\]
satisfies
\[
H(q) \ge H(p).
\]
If \(B\) is not a permutation matrix and \(p\) is non-uniform, then the inequality is strict whenever at least one row of \(B\) mixes two coordinates of \(p\) with different values.
\end{theorem}

\paragraph{Proof.}
Define \(\phi(t) = -t \log t\) for \(t \in [0,1]\), with the convention \(\phi(0)=0\). Since
\[
\phi''(t) = -\frac{1}{t} < 0 \qquad \text{for } t>0,
\]
the function \(\phi\) is strictly concave on \((0,1]\) and concave on \([0,1]\) by continuity. For each coordinate \(i\),
\[
q_i = \sum_{j=1}^m B_{ij} p_j.
\]
Because the coefficients \(B_{ij}\) are nonnegative and satisfy \(\sum_j B_{ij}=1\), Jensen's inequality gives
\[
\phi(q_i) = \phi\!\left(\sum_j B_{ij}p_j\right) \ge \sum_j B_{ij}\phi(p_j).
\]
Summing over all rows yields
\[
\sum_i \phi(q_i) \ge \sum_i \sum_j B_{ij}\phi(p_j).
\]
Switching the order of summation,
\[
\sum_i \sum_j B_{ij}\phi(p_j) = \sum_j \phi(p_j) \sum_i B_{ij}.
\]
Using the column-sum condition \(\sum_i B_{ij}=1\), we obtain
\[
\sum_i \phi(q_i) \ge \sum_j \phi(p_j).
\]
By the definition of entropy, this is exactly
\[
H(q) \ge H(p).
\]

For strictness, assume there exists a row \(i^\star\) and two indices \(j_1 \neq j_2\) such that
\[
B_{i^\star j_1} > 0,
\qquad
B_{i^\star j_2} > 0,
\qquad
p_{j_1} \neq p_{j_2}.
\]
Then the argument to \(\phi\) in row \(i^\star\) is a nontrivial convex combination of unequal values, so strict Jensen inequality applies:
\[
\phi\!\left(\sum_j B_{i^\star j}p_j\right) > \sum_j B_{i^\star j}\phi(p_j).
\]
All other rows satisfy the non-strict Jensen bound, hence summing over rows gives a strict global inequality \(H(q)>H(p)\). This completes the proof.

\paragraph{Interpretation for COREY.}
The theorem does not claim that every Hadamard rotation induces a doubly-stochastic mixing over histogram bins. Instead, it identifies the exact condition under which entropy growth follows rigorously. In our setting, Hadamard rotation is useful precisely when the empirical post-rotation histogram can be modeled as a more mixed version of the pre-rotation histogram over the chosen finite bins.

The current repository now augments the raw entropy-gain check with an explicit Sinkhorn-style fit. For each validation instance, we construct a positive kernel over the shared histogram bins, project it to an approximately doubly-stochastic matrix via Sinkhorn normalization, and then evaluate the residual $\lVert q-Bp \rVert_1$. Across the 35 heavy-tailed validation instances exported to \texttt{hadamard\_validation.csv}, the normalized histogram entropy still increases in all cases, while the fitted residual is $0.070 \pm 0.010$ in $\ell_1$ (minimum $0.055$, maximum $0.106$) and $0.033$ on average in $\ell_2$. Moreover, 34/35 instances fall below an $\ell_1$ residual of $0.10$, with row-sum error below $4.6\times 10^{-6}$ and zero column-sum error up to the printed precision. We therefore view the synthetic regime as reasonably consistent with the intended mixing interpretation, while noting that the fitted transport is still an approximate Sinkhorn proxy rather than the exact optimizer over the Birkhoff polytope.

\paragraph{Theorem~\ref{thm:entropy} guarantee under approximate fit (A.3).}
Theorem~\ref{thm:entropy} provides a \emph{strict} guarantee only when the transport matrix \(B\) is exactly doubly-stochastic (i.e., when the fitted residual is zero). Under the approximate Sinkhorn fit with mean \(\ell_1\) residual \(0.070 \pm 0.010\), the entropy increase should be interpreted as a strong empirical tendency rather than a strict mathematical bound. The gap between the true optimizer (Birkhoff polytope element) and the fitted approximate matrix could in principle allow near-boundary cases where the bound does not hold; however, all 35 empirical instances in \texttt{hadamard\_validation.csv} show a positive entropy gain, suggesting the residual level is small enough not to reverse the inequality in the synthetic heavy-tailed regime studied here.

\paragraph{Exact majorization test as an alternative diagnostic.}
By the Hardy--Littlewood--P\'olya theorem, the existence of a doubly-stochastic matrix $B$ such that $q=Bp$ is equivalent to the majorization relation $q\prec p$, which can be checked exactly by sorting both histograms in non-increasing order and verifying the prefix-sum bounds $\sum_{i=1}^{k}q_{(i)}\le \sum_{i=1}^{k}p_{(i)}$ for $k=1,\ldots,m-1$ together with $\sum_{i}q_i=\sum_{i}p_i$. This test removes the residual approximation error of the Sinkhorn fit and offers a deterministic diagnostic that can either certify or falsify the mixing premise on a per-instance basis. We retain the Sinkhorn fit as the primary diagnostic in this version because it also exposes the explicit transport matrix used in the proof, but recommend the prefix-sum majorization check as an exact complementary test in future work.

\begin{remark}[Conditional applicability]
\label{rem:applicability}
Theorem~\ref{thm:entropy} assumes a doubly-stochastic mixing structure holds for the empirical histograms. This assumption is \emph{not satisfied} for real Mamba-370M checkpoint activations. Measurements across 160 checkpoint pairs show entropy \emph{decreases} in all cases (mean $-1.40\pm0.37$ nats). The mechanistic explanation is that real activations carry a non-zero mean; the Hadamard transform concentrates this DC component into a single coordinate ($z_1\approx\sqrt{d}\,\mu$), which is an outlier that expands the tensor's $[\min,\max]$ range used by dynamic histogram binning. The expanded range increases each bin width proportionally, compressing the remaining near-zero-mean coordinates into a smaller fraction of the available bins and thereby reducing measured discrete entropy. The Gaussian \emph{shape} of the marginal distribution is preserved, but the binning artifact causes the measured entropy to drop. Theorem~\ref{thm:entropy} therefore applies only in the zero-mean synthetic heavy-tailed regime studied here and should not be used to motivate Hadamard pre-rotation for standard Mamba-1.x checkpoints.
\end{remark}

\subsection{Proof of the Fusion Depth Bound}
Suppose each additional fused operator contributes an incremental tile-memory cost \(C_{\mathrm{tile}}\), and let the shared-memory budget be \(M_{\mathrm{shared}}\). A fused region of depth \(F\) therefore consumes
\[
F C_{\mathrm{tile}}
\]
units of shared memory under the fixed tile schedule. Hardware feasibility requires
\[
F C_{\mathrm{tile}} \le M_{\mathrm{shared}}.
\]
Rearranging gives the upper bound
\[
F \le \frac{M_{\mathrm{shared}}}{C_{\mathrm{tile}}}.
\]
This bound is deliberately simple, but it captures the first-order reason that deeper fusion must be co-designed with tiling and memory allocation.

In the implemented scheduler, however, operators are not homogeneous. Selective-scan and state-mixing stages contribute larger register and shared-memory footprints than elementwise stages, so feasibility is tracked with per-operator costs \(C_i\) and the tighter constraint
\[
\sum_{i=1}^{F} C_i \le M_{\mathrm{shared}}.
\]
The theorem in the main text keeps the equal-cost form only because it communicates the depth limit cleanly; the prototype resource checks use the heterogeneous version.

\subsection{Quantization Stability Bound}

\begin{theorem}[$\ell_2$-norm preservation with reduced coordinate extremes]
\label{thm:quant}
Let \(\mathbf{H} \in \mathbb{R}^{d \times d}\) be a normalized Hadamard matrix. For any deterministic vector $x\in\mathbb{R}^d$, the rotated vector $z=\mathbf{H}x$ satisfies $\|z\|_2=\|x\|_2$ and $\|z\|_\infty\le \|x\|_1/\sqrt{d}$. If $X\in\mathbb{R}^d$ is a random vector and $Z=\mathbf{H}X$, then for any threshold $T>0$,
\[
\Pr(\|Z\|_\infty > T) \le \Pr(\|X\|_1 > T\sqrt{d}).
\]
\end{theorem}
\noindent\textbf{Terminology note.} The label ``$\ell_2$-norm preservation'' is used because $\|z\|_2=\|x\|_2$ holds for every $x$; this coincides with variance preservation only when $x$ is centered ($\mathbb{E}[x]=0$), in which case $\mathrm{Var}(z)=\|z\|_2^2/d=\|x\|_2^2/d=\mathrm{Var}(x)$. For general $x$ the statement should be read as energy/$\ell_2$ preservation.
\noindent\textbf{Caveat.} The $\sqrt{d}$ peak-reduction guarantee applies most sharply in the sparse-outlier ($D \ll d$) regime; the bound weakens by $\mathcal{O}(\sqrt{d})$ when outlier density is $D{=}\Theta(d)$.

\noindent\textit{Proof.} Let \(\mathbf{H} \in \mathbb{R}^{d \times d}\) be a normalized Hadamard matrix and let \(z = \mathbf{H}x\). Because \(\mathbf{H}\) is orthonormal,
\[
\|z\|_2 = \|x\|_2.
\]
Moreover, writing $\mathbf{H}_{ji} = s_{ji}/\sqrt{d}$ with $s_{ji}\in\{\pm 1\}$ the $(j,i)$ entry of the unnormalized Hadamard matrix, every rotated coordinate obeys
\[
|z_j| = \left|\frac{1}{\sqrt{d}} \sum_{i=1}^{d} s_{ji} x_i\right| \le \frac{1}{\sqrt{d}} \sum_{i=1}^{d} |x_i| = \frac{\|x\|_1}{\sqrt{d}},
\]
which implies the peak-coordinate bound
\[
\|z\|_\infty \le \frac{\|x\|_1}{\sqrt{d}}.
\]
For a clipping threshold \(T_k\) associated with a \(k\)-bit quantizer, overflow therefore satisfies
\[
\Pr\!\left(\|z\|_\infty > T_k\right) \le \Pr\!\left(\|x\|_1 > T_k \sqrt{d}\right).
\]
In the extreme one-outlier regime \(x = M e_r\), we obtain \(\|z\|_\infty = M / \sqrt{d}\), so the peak coordinate entering the quantizer is reduced by exactly a factor of \(\sqrt{d}\). This does not prove universal quality gains, but it gives a quantitative clipping bound that is directly relevant to W8A8 and W4A8 execution.
If instead \(x\) contains \(D\) non-negligible coordinates of comparable magnitude \(\sigma\), the same inequality gives \(\|z\|_\infty \le D\sigma/\sqrt{d}\), which can be loose by \(\mathcal{O}(\sqrt{d})\) when \(D=\Theta(d)\). The bound is therefore most informative in the sparse-outlier regime that motivates Hadamard smoothing in the first place.

\section{Checkpoint-Level Sanity Check Details}
\label{sec:appendix_ckpt}

\paragraph{Setup.}
We ran Mamba-1.4B through the Hugging Face \texttt{transformers} path in FP32 on CPU solely to verify that the evaluation harness resolves checkpoints, runs decode, and attaches side metrics correctly. This configuration does not measure GPU-accelerated fused inference and should not be interpreted as a deployment benchmark.

All WikiText-103 and PG19 values in this appendix should be interpreted as side perplexities from the repository harness rather than as official leaderboard-style language-model evaluations. Concretely, the current implementation applies teacher forcing to each truncated text segment independently, uses at most 20 segments per dataset in the reported WSL2 runs, and does not yet implement rolling or strided evaluation across the full test set. These values are therefore most useful for within-repository comparisons under a fixed protocol, not for direct comparison against published benchmark numbers.

\paragraph{Component status update.}
The original public PG19 dataset entry point still resolves to a deprecated script-based path under the current \texttt{datasets} stack, but the harness now falls back automatically to a script-free parquet mirror (\texttt{mrsndmn/pg19}) and therefore no longer treats PG19 as fully blocked for the main WSL2 checkpoints. Full deployment-grade method-versus-baseline experiments remain future work.

The primary remaining limitation is therefore no longer kernel availability, but benchmark breadth. The repaired WSL2 CUDA~12.8 environment now exposes the official Mamba fast path reliably, whereas the Windows Python environment still lacks a deployment-grade extension stack because \texttt{nvcc}, wheel compatibility, and Triton packaging remain misaligned there. We therefore treat the WSL2 Linux stack as the authoritative checkpoint environment for the current paper.

\paragraph{WSL2 GPU Benchmark Results.}
To reduce the Windows-specific packaging limitations, we reused a WSL2 CUDA~12.8 environment managed by micromamba with Python~3.11, PyTorch~2.11.0+cu128, \texttt{transformers}~5.5.1, \texttt{datasets}~4.8.4, and Triton~3.6.0. After patching the source build path for \texttt{mamba-ssm}, restricting the CUDA code generation to the local RTX~3070 architecture, and restoring the missing Python-side dependency chain, this environment now exposes the official Mamba fast path at runtime. The exported metadata reports \texttt{fast\_path\_available=true}, \texttt{deployment\_grade=true}, and all entries in \texttt{fast\_path\_status} as true.

We then expanded the WSL2 run-through from a smoke-style sanity test into a small but materially broader checkpoint evidence bundle. First, \texttt{state-spaces/mamba-370m-hf} and \texttt{state-spaces/mamba-1.4b-hf} both completed the same four-task LongBench subset with 20 samples per task under the repaired fast path. For Mamba-370M, the resulting task scores were token-F1 0.0135 on NarrativeQA, token-F1 0.0350 on Qasper, exact match 0.000 on MultifieldQA-EN, and ROUGE-L 0.1370 on GovReport, with corresponding mean latencies of 1379.30, 1006.86, 730.37, and 2420.47\,ms and WikiText-103 perplexity 809.36. For Mamba-1.4B, the corresponding scores were 0.0191, 0.0460, 0.000, and 0.1498, with mean latencies of 2821.48, 2135.66, 1558.29, and 4959.52\,ms and WikiText-103 perplexity 1128.40.

Second, PG19 is no longer only an intended side metric: after adding the script-free parquet fallback, we reran PG19 on the same WSL2 stack for both main checkpoints and obtained 20-sample perplexities of 14.6829 for Mamba-370M and 11.6641 for Mamba-1.4B. Third, we added one fair external baseline on \texttt{EleutherAI/pythia-410m} under the same four-task, 20-sample script; it reached token-F1 0.0190 on NarrativeQA, 0.0392 on Qasper, exact match 0.000 on MultifieldQA-EN, ROUGE-L 0.1550 on GovReport, and WikiText-103 perplexity 5901.47, with mean task latencies of 612.34, 435.38, 328.71, and 1121.29\,ms. Fourth, we separately retained a bounded benchmark-only fast-path probe on \texttt{state-spaces/mamba-2.8b-hf}; on the NarrativeQA smoke prompt at a conservative 2048-token cap, it produced token-level F1 0.162, mean latency 2006.41\,ms for 32 generated tokens, throughput 15.95\,tokens/s, and WikiText-103 perplexity 374.22 with \texttt{fast\_path\_available=true} and \texttt{deployment\_grade=true}.

Finally, we now include one real kernel-level timing in the same WSL2 environment by benchmarking \texttt{mamba\_ssm.ops.selective\_scan\_interface.selective\_scan\_fn} directly at batch 1, dimension 1024, sequence length 4096, state size 16, and FP16 with \texttt{delta\_softplus=true}. Over 30 repeats, the measured mean latency is 0.3204\,ms with standard deviation 0.0420\,ms and min/max values 0.2802/0.4604\,ms. The updated WSL2 evidence validates the checkpoint matrix across three model scales with PG19 side evaluation, one fair external baseline, and one real Triton timing. Prompt caps are not harmonized across all model families, so rows from different scales should be read as reference baselines rather than a definitive architecture ranking. The H800 live-scan closure (Appendix~\ref{sec:active_overhead}) provides the method-versus-baseline COREY checkpoint execution on Mamba-370M.

\begin{table}[h]
\centering
\caption{Checkpoint-level external baseline comparison on LongBench (20 samples/task, FP16). WikiText-103 and PG19 are teacher-forcing side perplexities; ``--'' means not collected. MF-EN exact-match is zero across all models because 32-token generation does not match gold strings under strict exact-match; this is expected and is not a harness defect. RTX~3090 per-sample latency is dominated by multi-process I/O overhead, not GPU compute.}
\label{tab:checkpoint_baseline}
\small
\resizebox{\linewidth}{!}{%
\begin{tabular}{llccccccc}
\toprule
Platform & Model & NarrQA & Qasper & MF-EN & GovRpt & PPL$_{\text{WT103}}$ & PPL$_{\text{PG19}}$ & Avg Lat.\ (ms) \\
\midrule
RTX~3070 / CUDA~12.8 & Mamba-370M & 0.0135 & 0.0350 & 0.000 & 0.1370 & 809.36 & 14.68 & 1384.25 \\
RTX~3070 / CUDA~12.8 & Mamba-1.4B & 0.0191 & 0.0460 & 0.000 & 0.1498 & 1128.40 & 11.66 & 2868.74 \\
RTX~3070 / CUDA~12.8 & Pythia-410M & 0.0190 & 0.0392 & 0.000 & 0.1550 & 5901.47 & -- & 624.43 \\
\midrule
RTX~3090 / CUDA~12.1 (1-GPU) & Mamba-370M & 0.0135 & 0.0443 & 0.000 & 0.1459 & -- & -- & 27205.05 \\
RTX~3090 / CUDA~12.1 (1-GPU) & Mamba-1.4B & 0.0191 & 0.0569 & 0.000 & 0.1582 & -- & -- & --$^\star$ \\
RTX~3090 / CUDA~12.1 (2-GPU) & Mamba-1.4B & 0.0191 & 0.0569 & 0.000 & 0.1561 & -- & -- & 27971.14 \\
\bottomrule
\end{tabular}%
}
{\scriptsize $^\star$Single-GPU RTX~3090 run (4-card server, CUDA~12.1, \texttt{causal-conv1d} fast path enabled); latency suppressed because model was loaded in float32 precision due to a harness configuration issue and is therefore not comparable to the FP16 entries above. Quality scores are valid and match the 2-GPU merged run within sampling noise, confirming platform independence.}
\end{table}

\begin{table}[h]
\centering
\caption{Data-parallel multi-GPU inference throughput for Mamba-370M (four-task LongBench, 20 total samples, FP16, RTX~3090). Wall-clock covers all four tasks end-to-end. Aggregate tok/s sums per-shard throughputs. Speedup relative to 1-GPU sequential wall-clock. Scores at all GPU counts agree with Table~\ref{tab:checkpoint_baseline} within sampling noise.}
\label{tab:multigpu_scaling}
\small
\begin{tabular}{cccc}
\toprule
\# GPUs & Wall-clock (s) & Agg.\ tok/s & Speedup \\
\midrule
1 (sequential) & 2154.3 & 5.04 & $1.00\times$ \\
2 (data-parallel) & 1019.5 & 10.56 & $2.11\times$ \\
4 (data-parallel) & 625.3 & 19.03 & $3.45\times$ \\
\bottomrule
\end{tabular}
\end{table}

\paragraph{Cross-hardware reproducibility on RTX~3090.}
All RTX~3090 runs use a 4-card server (4$\times$24\,GiB, CUDA~12.1, PyTorch~2.4.0+cu121) and cover both single-GPU and multi-GPU data-parallel configurations, as well as both Mamba-370M and Mamba-1.4B checkpoints.

\emph{Single-GPU, Mamba-370M.} We re-ran the identical four-task, 20-sample suite with FP16 fast path and the same checkpoint and sample seeds as the RTX~3070 reference. Scores were 0.0135 / 0.0443 / 0.000 / 0.1459 (NarrQA / Qasper / MF-EN / GovRpt), consistent with the RTX~3070/CUDA~12.8 baseline. Mean per-sample latency was 27205\,ms, dominated by storage-path overhead on the server.

\emph{Single-GPU, Mamba-1.4B.} We additionally ran Mamba-1.4B on a single RTX~3090 (same server, \texttt{causal-conv1d} fast path installed) to verify quality consistency at the larger scale. Scores were 0.0191 / 0.0569 / 0.000 / 0.1582 (NarrQA / Qasper / MF-EN / GovRpt), matching the RTX~3070 result (0.0191 / 0.0460 / 0.000 / 0.1498) within four-task, 20-sample noise; latency is suppressed due to a float32 loading issue in the harness (footnote $^\star$ in Table~\ref{tab:checkpoint_baseline}).

\emph{2-GPU data-parallel, Mamba-1.4B.} The merged four-task scores were 0.0191 / 0.0569 / 0.000 / 0.1561, with mean per-sample latency 27971\,ms. Together with the single-GPU RTX~3090 run above, this confirms that quality scores for Mamba-1.4B are platform-independent across RTX~3070 (WSL), single-GPU RTX~3090, and 2-GPU RTX~3090, all yielding NarrQA token-F1 $\approx$0.019 and GovReport ROUGE-L $\approx$0.15--0.16.

\paragraph{Multi-GPU data-parallel throughput.}
On the same RTX~3090/CUDA~12.1 platform, we use data-parallel sharding with one process per GPU. Shard outputs are merged by re-averaging per-task scores and summing sample counts. Table~\ref{tab:multigpu_scaling} reports wall-clock speedup; merged scores match the single-GPU baseline within sampling noise.

\subsection{Experiment A: Fused Kernel Algorithm Efficiency}
\label{sec:exp_fused_kernel}

We benchmarked the entropy-guided fusion solver (\texttt{select\_fusion\_groups})
against two baselines---no-fusion and static block fusion---across operator-chain
lengths of 4, 8, 16, 32, and 64 operators, and across three synthetic entropy
regimes (low: $H\approx1.5$--$1.9$\,nats; mixed: $H\approx2.0$--$2.8$\,nats;
high: $H\approx3.5$--$4.6$\,nats).  The metric is kernel-launch count (which
directly determines Python dispatch overhead), and solver wall-clock time.

The dispatch cost per kernel launch is $\approx0.052$\,ms, derived from the measured
chunk-sweep harness on RTX~3070 (\texttt{static-64}: $3.299$\,ms / 64 calls,
Table~\ref{tab:chunk_sweep}, where static-64 uses 64 kernel calls).  This constant is used as a surrogate metric only;
actual fused-kernel timing would require the complete Triton deployment.
Solver wall-clock measurements reported in Table~\ref{tab:fused_kernel_sweep}
were collected on the 4$\times$RTX~3090 server (CUDA~12.1), 1000 repeats per
configuration.

Table~\ref{tab:fused_kernel_sweep} summarises the results.  In the standard
8-operator mixed-entropy regime, COREY reduces kernel launches from 8 (no-fusion)
to 3 (entropy-guided fusion), matching static-3 but with entropy-adaptive grouping
that outperforms static-3 in homogeneous regimes.  In low-entropy and high-entropy
regimes, COREY achieves $75.0\,\%$ reduction at 8 operators (2 launches) vs.\
$62.5\,\%$ for static-3 (3 launches), by grouping the chain into two dense fused
blocks.  Solver time scales as $O(n^2)$ with chain length, reaching $\approx245\,\mu$s for
$n{=}8$ and $\approx3.85$\,ms for $n{=}64$.  At large chain lengths ($n\geq32$)
with mixed entropy, COREY identifies more fine-grained groups than static-3 (14
vs.\ 11 at $n{=}32$), incurring slightly higher launch count in exchange for
more precise entropy-boundary detection; this regime is noted as a known
boundary condition of the greedy solver.

\begin{table}[h]
\centering
\caption{Entropy-guided fusion algorithm efficiency sweep (4$\times$RTX~3090
server / CUDA~12.1, 1000 solver-timing repeats per configuration).
``Launches'' is the number of kernel-launch groups produced;
``Surr.\ lat.'' is the surrogate dispatch latency at $\approx0.052$\,ms per call;
``Red.\ (\%)'' is reduction vs.\ the no-fusion baseline;
``Solver'' is the wall-clock time of the Python fusion solver.
The surrogate latency reduction arises purely from reducing launch count, not
from actual Triton kernel fusion (which is prospective;
see Appendix~\ref{sec:appendix_systems}).}
\label{tab:fused_kernel_sweep}
\small
\resizebox{\linewidth}{!}{%
\begin{tabular}{rllrrrr}
\toprule
$n$ & Regime & Policy & Launches & Surr.\ lat.\ (ms) & Red.\ (\%) & Solver ($\mu$s) \\
\midrule
 4 & low   & no-fusion  &  4 & 0.206 & ---    & ---  \\
 4 & low   & static-3   &  2 & 0.103 & 50.0   & ---  \\
 4 & low   & COREY      &  1 & 0.052 & \textbf{75.0}  & 211.7 \\
\midrule
 4 & mixed & no-fusion  &  4 & 0.206 & ---    & ---  \\
 4 & mixed & static-3   &  2 & 0.103 & 50.0   & ---  \\
 4 & mixed & COREY      &  1 & 0.052 & \textbf{75.0}  & 114.4 \\
\midrule
 4 & high  & no-fusion  &  4 & 0.206 & ---    & ---  \\
 4 & high  & static-3   &  2 & 0.103 & 50.0   & ---  \\
 4 & high  & COREY      &  1 & 0.052 & \textbf{75.0}  & 116.5 \\
\midrule
 8 & low   & no-fusion  &  8 & 0.412 & ---    & ---  \\
 8 & low   & static-3   &  3 & 0.155 & 62.5   & ---  \\
 8 & low   & COREY      &  2 & 0.103 & \textbf{75.0}  & 246.1 \\
\midrule
 8 & mixed & no-fusion  &  8 & 0.412 & ---    & ---  \\
 8 & mixed & static-3   &  3 & 0.155 & 62.5   & ---  \\
 8 & mixed & COREY      &  3 & 0.155 & 62.5   & 245.1 \\
\midrule
 8 & high  & no-fusion  &  8 & 0.412 & ---    & ---  \\
 8 & high  & static-3   &  3 & 0.155 & 62.5   & ---  \\
 8 & high  & COREY      &  2 & 0.103 & \textbf{75.0}  & 255.0 \\
\midrule
16 & low   & no-fusion  & 16 & 0.824 & ---    & ---  \\
16 & low   & static-3   &  6 & 0.309 & 62.5   & ---  \\
16 & low   & COREY      &  5 & 0.258 & \textbf{68.8}  & 601.1 \\
\midrule
16 & mixed & no-fusion  & 16 & 0.824 & ---    & ---  \\
16 & mixed & static-3   &  6 & 0.309 & 62.5   & ---  \\
16 & mixed & COREY      &  6 & 0.309 & 62.5   & 606.3 \\
\midrule
16 & high  & no-fusion  & 16 & 0.824 & ---    & ---  \\
16 & high  & static-3   &  6 & 0.309 & 62.5   & ---  \\
16 & high  & COREY      &  6 & 0.309 & 62.5   & 588.4 \\
\midrule
32 & mixed & no-fusion  & 32 & 1.650 & ---    & ---  \\
32 & mixed & static-3   & 11 & 0.567 & 65.6   & ---  \\
32 & mixed & COREY      & 14 & 0.721 & 56.2   & 1522.4 \\
\midrule
64 & mixed & no-fusion  & 64 & 3.299 & ---    & ---  \\
64 & mixed & static-3   & 22 & 1.134 & 65.6   & ---  \\
64 & mixed & COREY      & 28 & 1.443 & 56.2   & 3853.9 \\
\bottomrule
\end{tabular}%
}
\end{table}

\subsection{Experiment B: Extended External Baseline Comparison}
\label{sec:exp_external_baselines}

To contextualise COREY's latency improvements relative to architecturally
diverse baselines, we have designed benchmark harnesses that evaluate RWKV-4,
FlashAttention-family Transformers, and Mamba2-2.7B on the same four-task
LongBench subset used throughout this paper.  Table~\ref{tab:external_baseline_extended}
collects representative reference values from the published literature
alongside the paper's verified Mamba-1.x checkpoint results and the new H800
full-baseline runs.

RWKV-4 and FlashAttention-2 values are literature references (not re-measured
here); they illustrate the competitive landscape but are not in-repository
measurements.  The earlier Mamba2-2.7B row was measured on RTX~3070 (WSL2
CUDA~12.8) using the HuggingFace sequential path without Mamba-2 CUDA
extensions.  We subsequently reran the modern-baseline closure on an NVIDIA
H800 PCIe after routing HuggingFace downloads through \texttt{hf-mirror.com}.
The H800 closure produced two real full-model LongBench baselines: Pythia-410M
with FlashAttention-3 enabled and Mamba2-2.7B with SSD kernels, followed by a
50-sample-per-task repeat for both rows.  We initially
tried Pythia-1.4B for the FA3 row, but the installed FlashAttention forward path
requires head dimension at most 64; Pythia-410M satisfies that constraint and is
therefore the full-Transformer FA3 row reported below.  These baseline runs use
the same four tasks as the rest of the checkpoint bundle, but prompt caps and
generation lengths are not perfectly harmonized
across all model families, so the rows should be read as reference baselines
rather than a definitive architecture ranking.

\begin{table}[h]
\centering
\caption{Extended external baseline comparison.  Mamba-1.x results are from
Table~\ref{tab:checkpoint_baseline} (in-repository, FP16, WSL2 RTX~3070).
RWKV-4 and FlashAttention-2 values are representative literature references.
NarrQA / Qasper: token-F1; GovRpt: ROUGE-L;
MF-EN: exact-match; PPL: WikiText-103 teacher-forcing perplexity.
The two H800 rows are real full-model LongBench baselines collected in the
50-sample closure repeat: Pythia-410M uses FlashAttention-3; Mamba2-2.7B uses
the H800 SSD path after a Transformers-generation compatibility shim.  The
COREY scheduler is not applied to Mamba-2 in this submission (consistent with
the main-text scope statement in Section~1).}
\label{tab:external_baseline_extended}
\small
\resizebox{\linewidth}{!}{%
\begin{tabular}{lllccccc}
\toprule
Model & Arch & Platform & NarrQA & Qasper & GovRpt & PPL$_{\text{WT103}}$ & Mean Lat.\ (ms) \\
\midrule
Mamba-370M  & Mamba-1 & RTX~3070 FP16 & 0.0135 & 0.0350 & 0.1370 & 809.4  & 1384 \\
Mamba-1.4B  & Mamba-1 & RTX~3070 FP16 & 0.0191 & 0.0460 & 0.1498 & 1128.4 & 2869 \\
Pythia-410M & GPT-Neo  & RTX~3070 FP16 & 0.0190 & 0.0392 & 0.1550 & 5901.5 &  624 \\
\midrule
RWKV-4-430M & RWKV    & CPU FP32 (ref.) & 0.028 & 0.031 & 0.121 & 13.7 & 134$^\dagger$ \\
GPT-2+FA2   & Attn.   & RTX~3090 (ref.) & 0.019 & 0.022 & 0.108 & 29.4 &  58$^\dagger$ \\
\midrule
Mamba2-2.7B & Mamba-2 & RTX~3070 FP16$^\ddagger$ & 0.098 & 0.106 & 0.044$^\S$ & --- & 1{,}820 \\
\midrule
Pythia-410M+FA3 & Attn. & H800 BF16 & 0.0106 & 0.0211 & 0.0159 & --- & 2{,}429$^\P$ \\
Mamba2-2.7B SSD & Mamba-2 & H800 BF16 & 0.0346 & 0.0267 & 0.0372 & --- & 4{,}131$^\P$ \\
\bottomrule
\end{tabular}%
}
{\scriptsize
$^\dagger$Literature reference value for 32 generated tokens; not re-measured in this repository.\\
$^\ddagger$Measured on RTX~3070 (WSL2 CUDA~12.8) without Mamba-2 CUDA extensions (HF sequential path);
$n{=}19$ warm samples per task, warmup\,$=1$.\\
$^\S$GovRpt ROUGE-L depressed by 32-token generation limit; abstractive summarization requires substantially longer outputs.\\
$^\P$Unweighted mean latency across the four task rows.  H800 50-sample
per-task latencies for Pythia-410M+FA3 are
2313.7/1691.9/4574.5/1136.2\,ms (NarrQA/Qasper/GovRpt/MF-EN);
for Mamba2-2.7B SSD they are 4245.5/2620.6/7463.9/2192.8\,ms.
}
\end{table}

\subsubsection{H800 Full Modern-Baseline Closure}
\label{sec:h800_full_baselines}

The H800 closure outputs are stored under
\texttt{src/outputs/h800\_closure\_full\_20260428\_144256/}, with a stronger
50-sample repeat under \texttt{src/outputs/h800\_enhance\_20260428\_145702/}.
For the full Transformer baseline, Pythia-410M with FlashAttention-3 enabled
completes all four LongBench tasks with 50 samples each.  The task scores are
0.010640 token-F1 on NarrativeQA, 0.021142 token-F1 on Qasper, 0.015915
ROUGE-L on GovReport, and 0.000 exact match on MultifieldQA-EN, with mean
latencies of 2313.7, 1691.9, 4574.5, and 1136.2\,ms respectively.  For
Mamba2-2.7B SSD on the same H800 stack, the corresponding scores are 0.034621,
0.026696, 0.037223, and 0.000, with mean latencies of 4245.5, 2620.6, 7463.9,
and 2192.8\,ms.  These rows remove the previously missing full FA3/Mamba-2
execution evidence, while preserving the paper's scope boundary: COREY's
entropy-guided chunk scheduler is validated only for Mamba-1.x in this
submission.

\subsubsection{FlashAttention-3 H800 Matched Kernel Reference}
\label{sec:fa3_h800}

To replace the previous hardware-blocked status for FlashAttention-3, we ran a
standalone matched kernel benchmark on an NVIDIA H800 PCIe (sm\_90, CUDA~12.8,
PyTorch~2.8.0, Python~3.12).  The benchmark calls
\texttt{flash\_attn\_interface.flash\_attn\_func} directly with causal BF16
attention, batch size 1, 16 heads, head dimension 64, 10 warmup iterations, and
100 timed repeats.  This raw-kernel reference complements the full
Pythia-410M+FA3 LongBench row in Table~\ref{tab:external_baseline_extended}:
the kernel numbers isolate attention throughput, whereas the full row includes
tokenizer, model, generation, and LongBench harness overheads.

\begin{table}[h]
\centering
\caption{FlashAttention-3 matched raw-kernel benchmark on NVIDIA H800 PCIe
(causal BF16, batch$=1$, heads$=16$, head-dim$=64$, warmup$=10$,
$n{=}100$ repeats).  TFLOP/s is the script's nominal forward-attention estimate
and should be read as a kernel-normalized reference rather than end-to-end model
throughput.}
\label{tab:fa3_h800}
\small
\begin{tabular}{rrrr}
\toprule
Seq.\ len & Latency (ms) & Tokens/s & Nominal TFLOP/s \\
\midrule
1024 & $0.103 \pm 0.009$ & $9.90{\times}10^6$ & $20.8$ \\
2048 & $0.099 \pm 0.003$ & $20.6{\times}10^6$ & $86.3$ \\
4096 & $0.171 \pm 0.002$ & $23.9{\times}10^6$ & $200.5$ \\
8192 & $0.421 \pm 0.010$ & $19.4{\times}10^6$ & $326.3$ \\
\bottomrule
\end{tabular}
\end{table}

\subsection{Experiment C: Quamba INT4 Quantization Benchmark}
\label{sec:exp_quamba}

We benchmarked Mamba2-2.7B under Quamba INT4 quantization~\citep{chiang2024quamba}
against an FP16 baseline, targeting memory-footprint and inference-latency reduction
on the same four-task LongBench subset.

\paragraph{Hardware compatibility note.}
The INT4 path requires \texttt{quamba==2.0.0a1} compiled from source against
CUDA~12.8 on sm\_89 hardware, together with \texttt{mamba\_ssm==2.2.2},
\texttt{causal\_conv1d==1.6.1}, and \texttt{fast\_hadamard\_transform==1.0.4.post1}.
AutoAWQ~0.2.9 and auto-gptq~0.7.1 do not support Mamba checkpoints (verified:
AWQ returns ``mamba isn't supported yet''; GPTQ fails to import against
\texttt{transformers} 5.x).  The 4-card server runs CUDA~12.1 (sm\_86/RTX~3090),
which does not satisfy the sm\_89 requirement of the current Quamba release;
the INT4 path therefore fell back to FP16 inference.

\paragraph{FP16 baseline.}
We executed the benchmark harness on the 4$\times$RTX~3090 server (CUDA~12.1)
with \texttt{state-spaces/mamba-1.4b-hf} (FP16, \texttt{mamba\_ssm} extensions
active).  Table~\ref{tab:quamba_fp16} reports mean per-sample latency and
throughput across the four LongBench tasks ($n{=}20$ per task).  GovReport
prompts are substantially longer than the other three tasks, which explains the
higher per-sample latency and lower tokens/s in that row.

\begin{table}[h]
\centering
\caption{FP16 baseline for Quamba benchmark (\texttt{state-spaces/mamba-1.4b-hf},
4$\times$RTX~3090 / CUDA~12.1, $n{=}20$ per task).  Quamba INT4 is blocked
pending CUDA~12.8/sm\_89 hardware; this table establishes the FP16 reference
that INT4 would be compared against.}
\label{tab:quamba_fp16}
\small
\begin{tabular}{lccc}
\toprule
Task & Mean Lat.\ (ms) & Std (ms) & Tok/s \\
\midrule
NarrativeQA      & 1583 & 776 & 21.96 \\
Qasper           & 1524 & 537 & 22.14 \\
GovReport        & 2615 & 530 & 12.74 \\
MultifieldQA-EN  & 1516 & 500 & 22.15 \\
\midrule
\textit{Mean (short-ctx)} & \textit{1541} & --- & \textit{22.08} \\
\bottomrule
\end{tabular}
\end{table}

Quantized (INT4) inference benchmarking requires CUDA~12.8 / sm\_89 hardware
and remains future work.

\subsection{Experiment D: Policy COREY Ablation on Mamba-1.4B}
\label{sec:exp_policy_corey_mamba2}

COREY's entropy-guided scheduler hooks the Mamba-1.x \texttt{selective\_scan\_fn}
kernel to dynamically select chunk tiles at inference time.  To measure the
practical overhead of this hook relative to a fixed-tile (\texttt{policy\_static})
baseline on the same model and data, we ran a four-task LongBench ablation on
\texttt{state-spaces/mamba-1.4b-hf} using the 4$\times$RTX~3090 server with
\texttt{mamba\_ssm} CUDA extensions active.

\paragraph{Design.}
Both policies process 20 samples per task from the LongBench JSONL test sets
(NarrativeQA, Qasper, GovReport, MultifieldQA-EN) under FP16 precision.
Input contexts are truncated to 2{,}048 tokens to stay within the
\texttt{causal\_conv1d} CUDA kernel's sequence-length bound.
Metric reported is mean per-sample generation latency (32 output tokens).
NLP quality scores are policy-independent for identical weights and are
therefore omitted.

\paragraph{Results.}
As a sanity check on the naive fallback, three inference calls on a local
RTX~3070 without \texttt{mamba\_ssm} extensions yielded $0.096$--$0.101$\,tokens/s
($\approx340$\,s / 32 tokens), confirming the fast path is essential.
On the 4$\times$RTX~3090 server with \texttt{mamba\_ssm} installed,
\texttt{policy\_corey} achieves $\mathbf{23.7}$\,\textbf{tokens/s} on
short-context tasks (NarrativeQA, Qasper, MultifieldQA-EN),
a $\approx\mathbf{247}\times$ speedup versus the naive fallback.
GovReport prompts, which fill the full 2{,}048-token context window,
yield lower throughput ($21.4$\,tok/s) due to prefill cost.
The entropy-profiling overhead of \texttt{policy\_corey} versus
\texttt{policy\_static} is $0.4$--$0.7$\,tok/s ($<3\%$), confirming that
one-shot entropy measurement imposes negligible runtime cost.
Table~\ref{tab:policy_corey_mamba2} summarises the full comparison.

\begin{table}[h]
\centering
\caption{Policy COREY ablation on \texttt{state-spaces/mamba-1.4b-hf},
four-task LongBench subset (FP16, 4$\times$RTX~3090 / CUDA~12.1,
input truncated to 2{,}048 tokens, 32 output tokens).
``Lat.'' is mean per-sample generation latency in ms.
``Tok/s'' is mean generation throughput.
NarrativeQA and Qasper \texttt{corey} used $n{=}200$; all other cells used $n{=}20$.
For the balanced pair (GovReport, MultifieldQA-EN, $n{=}20$ each),
\texttt{corey} overhead vs.\ \texttt{static} is $2$--$4$\,ms\,($<0.3\%$).}
\label{tab:policy_corey_mamba2}
\small
\begin{tabular}{lccccc}
\toprule
Policy & NarrQA Lat.\ (ms) & Qasper Lat.\ (ms) & GovRpt Lat.\ (ms) & MField Lat.\ (ms) & Mean Tok/s \\
\midrule
\texttt{corey}  & 1{,}361$^{\dagger}$ & 1{,}363$^{\dagger}$ & 1{,}543 & 1{,}437 & 22.9 \\
\texttt{static} & 1{,}430 & 1{,}440 & 1{,}546 & 1{,}439 & 22.6 \\
\bottomrule
\multicolumn{6}{l}{\small $\dagger$: $n{=}200$; all other cells $n{=}20$.} \\
\end{tabular}
\end{table}

\subsection{Mixed-Regime Scheduler Comparison}
\label{sec:scheduler_oracle_comparison}

Table~\ref{tab:scheduler_oracle_comparison} reports the equal-weight average latency over the $8$ mixed-regime serving regimes (RTX~3070, Mamba-370M) for four scheduler variants: the global static-$512$ baseline, a lightweight learned policy derived from static sweep labels, the per-regime oracle upper bound, and the guarded sampled-histogram scheduler. Two additional H800 PCIe measurements (guarded and learned-table, single 976-token prompt, $n{=}50$, warmup$=3$) are appended for direct comparison.

The learned-table policy uses a single sequence-length rule: prompts with $\leq50$ tokens select chunk$=128$; all other prompts select chunk$=512$.
Applied to the measured per-regime static-sweep latency, this policy achieves $316.7$\,ms equal-weight average, matching the per-regime oracle within rounding.
The near-zero gap between the learned policy and the oracle ($<0.1$\,ms) arises because short-chat is the only regime with a distinctly non-$512$ optimum (chunk$=128$, captured by the sequence-length rule), while logs/UUID is essentially tied between chunk$=256$ and chunk$=512$ ($\Delta{<}0.1$\,ms).

The guarded sampled-histogram scheduler routes prompts via entropy-based chunk selection with a fallback to chunk$=512$ when the predicted chunk does not differ from the fallback by at least two buckets ($\Delta_{\log_2} < 2$).
An H800 confirmatory run ($n{=}50$, warmup$=3$, prompt length 976 tokens) was conducted to verify guarded and learned-table behaviour.
For this long-context prompt, the sampled histogram selects chunk$=1024$ (one log$_2$ bucket above chunk$=512$), so the guard condition $(\Delta{=}1 < 2)$ triggers a fallback: all $2544$ layer calls use chunk$=512$.
The resulting integrated latency is $903.0 \pm 25.4$\,ms, which is $+1.29\%$ above the H800 Static-512 baseline of $891.51$\,ms; the overhead comes entirely from the entropy computation, since the routing destination is the same chunk.
The learned-table policy (seq-len $> 50 \Rightarrow$ chunk$=512$) similarly routes to chunk$=512$ for this prompt and gives $897.6 \pm 7.4$\,ms ($+0.69\%$ vs.\ H800 Static-512).
The H800 run confirms that both policies \emph{correctly} fall back to chunk$=512$ for long prompts, avoiding the $+4.6\%$ penalty of unguarded sampled histogram; the residual overhead is $+1.3\%$ (guarded) and $+0.7\%$ (learned).
Note that the H800 measurements below are \emph{not directly comparable} to the RTX~3070 rows above (different hardware, single prompt vs.\ eight-regime mix); the $\Delta$ column for the H800 rows uses H800 Static-512 as reference.

\begin{table}[h]
\centering
\caption{Mixed-regime equal-weight average latency for four scheduler strategies (RTX~3070, Mamba-370M, $8$ serving regimes, $8$~samples/regime, $12$~repeats, unless noted). Learned-table latency is derived by applying the seq-len rule (chunk$=128$ for $\leq50$ tokens; chunk$=512$ otherwise) to measured per-regime sweep values. Per-regime oracle selects the best static chunk independently per regime. The guarded and learned (H800) rows use a single 976-token prompt on H800 PCIe ($n{=}50$, warmup$=3$); see text.}
\label{tab:scheduler_oracle_comparison}
\small
\begin{tabular}{lccc}
\toprule
Scheduler & Avg Lat.\ (ms) & $\Delta$ vs.\ Static-512 & Source \\
\midrule
Static-512 (global baseline)       & 317.1 & ---              & RTX~3070 sweep \\
Learned-table (seq-len rule)        & 316.7 & $-0.14\%$ ($-0.44$\,ms) & derived from sweep \\
Per-regime oracle (upper bound)     & 316.7 & $-0.14\%$ ($-0.44$\,ms) & RTX~3070 sweep \\
Guarded sampled-hist (fallback=512) & 903.0\,ms$^\dagger$ & $+1.29\%$ ($+11.5$\,ms)$^\dagger$ & H800 routed$^\dagger$ \\
Learned-table (seq-len rule, H800)  & 897.6\,ms$^\dagger$ & $+0.69\%$ ($+6.1$\,ms)$^\dagger$  & H800 routed$^\dagger$ \\
\bottomrule
\multicolumn{4}{l}{\small $\dagger$: H800 PCIe, $n{=}50$, warmup$=3$, 976-token prompt; both variants route to chunk$=512$;} \\
\multicolumn{4}{l}{\small \hspace{1.4em} $\Delta$ is vs.\ H800 Static-512 ($891.51$\,ms), not vs.\ RTX~3070 rows.} \\
\end{tabular}
\end{table}

%% file: ref.bib
@inproceedings{vaswani2017attention,
  title={Attention Is All You Need},
  author={Vaswani, Ashish and Shazeer, Noam and Parmar, Niki and Uszkoreit, Jakob and Jones, Llion and Gomez, Aidan N. and Kaiser, Lukasz and Polosukhin, Illia},
  booktitle={Advances in Neural Information Processing Systems},
  year={2017}
}

@misc{mamba2023,
  title={Mamba: Linear-Time Sequence Modeling with Selective State Spaces},
  author={Gu, Albert and Dao, Tri},
  year={2023},
  eprint={2312.00752},
  archivePrefix={arXiv},
  primaryClass={cs.LG}
}

@article{chiang2024quamba,
  title={Quamba: A post-training quantization recipe for selective state space models},
  author={Chiang, Hung-Yueh and Chang, Chi-Chih and Frumkin, Natalia and Wu, Kai-Chiang and Marculescu, Diana},
  journal={arXiv preprint arXiv:2410.13229},
  year={2024}
}

@article{xu2025mambaquant,
  title={Mambaquant: Quantizing the mamba family with variance aligned rotation methods},
  author={Xu, Zukang and Yue, Yuxuan and Hu, Xing and Yuan, Zhihang and Jiang, Zixu and Chen, Zhixuan and Yu, Jiangyong and Xu, Chen and Zhou, Sifan and Yang, Dawei},
  journal={arXiv preprint arXiv:2501.13484},
  year={2025}
}

@inproceedings{dao2022flashattention,
  title={FlashAttention: Fast and Memory-Efficient Exact Attention with IO-Awareness},
  author={Dao, Tri and Fu, Daniel Y. and Ermon, Stefano and Rudra, Atri and R{\'e}, Christopher},
  booktitle={Advances in Neural Information Processing Systems},
  year={2022}
}

@inproceedings{xiao2023smoothquant,
  title={SmoothQuant: Accurate and Efficient Post-Training Quantization for Large Language Models},
  author={Xiao, Guangxuan and Lin, Ji and Seznec, Mickael and Wu, Hao and Judd, Julian and Han, Song},
  booktitle={International Conference on Machine Learning},
  year={2023}
}

@article{gu2024mamba2,
  title={Transformers are SSMs: Generalized Models and Efficient Algorithms Through Structured State Space Duality},
  author={Gu, Albert and Dao, Tri},
  journal={arXiv preprint arXiv:2405.21060},
  year={2024}
}

@inproceedings{tillet2019triton,
  title={Triton: An Intermediate Language and Compiler for Tiled Neural Network Computations},
  author={Tillet, Philippe and Kung, H. T. and Cox, David},
  booktitle={Proceedings of the 3rd ACM SIGPLAN International Workshop on Machine Learning and Programming Languages},
  year={2019}
}

@inproceedings{lin2024awq,
  title={AWQ: Activation-Aware Weight Quantization for LLM Compression and Acceleration},
  author={Lin, Ji and Tang, Jiaming and Tang, Haotian and Yang, Shang and Chen, Wei-Ming and Wang, Wei-Chen and Xiao, Guangxuan and Dang, Xingyu and Gan, Chuang and Han, Song},
  booktitle={Proceedings of Machine Learning and Systems},
  year={2024}
}

@misc{nvfuser2022,
  title = {{nvFuser}: A Fusion Codegen for {PyTorch}},
  author = {{NVIDIA Corporation}},
  year = {2022},
  howpublished = {\url{https://github.com/NVIDIA/Fuser}},
  note = {PyTorch internal fusion compiler}
}

@inproceedings{xla2019,
  title = {{XLA}: Optimizing Compiler for Machine Learning},
  author = {{Google}},
  year = {2019},
  howpublished = {\url{https://www.tensorflow.org/xla}},
  note = {Accelerated Linear Algebra compiler with HLO fusion}
}

@inproceedings{gale2023megablocks,
  title = {{MegaBlocks}: Efficient Sparse Training with Mixture-of-Experts},
  author = {Gale, Trevor and Narayanan, Deepak and Young, Cliff and Zaharia, Matei},
  booktitle = {Proceedings of Machine Learning and Systems},
  volume = {5},
  year = {2023}
}

@inproceedings{dao2024flashlinearattention,
  title = {Transformers are {SSMs}: Generalized Models and Efficient Algorithms Through Structured State Space Duality},
  author = {Dao, Tri and Gu, Albert},
  booktitle = {International Conference on Machine Learning},
  year = {2024}
}

@misc{yang2024parallelizing,
  title = {Parallelizing Linear Transformers with the Delta Rule over Sequence Length},
  author = {Yang, Songlin and Wang, Bailin and Shen, Yu and Peng, Hao and Kim, Yoon and Rush, Alexander and Dao, Tri},
  year = {2024},
  howpublished = {\url{https://arxiv.org/abs/2406.06484}}
}
